\newif\ifcolt
\colttrue

\ifcolt
\PassOptionsToPackage{breaklinks=true}{hyperref}
\PassOptionsToPackage{algo2e, ruled, noline, linesnumbered, noend}{algorithm2e}
\documentclass{colt2017} % Include author names
\else
\documentclass{article}
\fi

\ifcolt
\else
\usepackage{amsthm}
\usepackage[utf8]{inputenc}
\usepackage[margin=1in]{geometry}
\newtheorem{lemma}{Lemma}
\newtheorem{theorem}{Theorem}
\newtheorem{corollary}{Corollary}

\newtheorem{example}{Example}
\fi

\newtheorem*{remark*}{Remark}

\usepackage{amssymb}
\usepackage{amsmath}
\usepackage{footnote}
\usepackage{natbib}
\usepackage{url}
\usepackage{breakcites}
\usepackage{xspace}
\usepackage{comment}
\usepackage{times}

\newcommand{\argmin}[1]{\underset{#1}{\text{argmin}}\;}
\newcommand{\bregman}[3]{D_{#1}(#2, #3)}
\newcommand{\inner}[2]{\langle #1, #2 \rangle}
\newcommand{\norm}[1]{\left\|{#1}\right\|}
\newcommand{\paren}[1]{\left({#1}\right)}

\newcommand{\calD}{\mathcal{D}}

\newcommand{\calP}{\mathcal{P}}
\newcommand{\calX}{\mathcal{X}}
\newcommand{\calR}{\mathcal{R}}

\newcommand{\calE}{\mathcal{E}}

\newcommand{\R}{\mathbb{R}}
\newcommand{\N}{\mathbb{N}}
\newcommand{\E}{\mathbb{E}}

\newcommand{\removed}[1]{}
\newcommand{\order}{\ensuremath{\mathcal{O}}}
\newcommand{\otil}{\ensuremath{\widetilde{\mathcal{O}}}}
\newcommand{\bc}{\boldsymbol{c}}
\newcommand{\bp}{\boldsymbol{p}}
\newcommand{\bq}{\boldsymbol{q}}
\newcommand{\bu}{\boldsymbol{u}}
\newcommand{\bW}{\boldsymbol{W}}
\newcommand{\boldeta}{\boldsymbol{\eta}}
\newcommand{\ptilde}{\tilde{p}}
\newcommand{\pbar}{\bar{p}}
\newcommand{\cbar}{\bar{c}}
\newcommand{\loss}{\ell}
\newcommand{\bloss}{\boldsymbol{\loss}}
\newcommand{\bpbar}{{\boldsymbol{\bar{p}}}}
\newcommand{\bptilde}{{\boldsymbol{\tilde{p}}}}
\newcommand{\one}{\boldsymbol{1}}
\newcommand{\zero}{\boldsymbol{0}}
\newcommand{\basis}{\boldsymbol{e}}

\newcommand{\base}[1]{{{\cal{B}}_{#1}}}

\newcommand{\reg}{{{\cal{R}}}}

\newcommand{\alg}{\textsc{Corral}\xspace}

\newcommand{\LOMD}{\textsc{Log-Barrier-OMD}\xspace}
\newcommand{\Rel}{\textsc{Rel}}
\newcommand{\rad}{\epsilon}
\newcommand{\vrad}{{\bf \rad}}

\makesavenoteenv{algorithm} 
\makesavenoteenv{algorithmic}

\newcommand{\scale}{\rho}

\ifcolt
\title[Corralling a Band of Bandit Algorithms]{Corralling a Band of Bandit Algorithms}  
\coltauthor{\Name{Alekh Agarwal} \Email{alekha@microsoft.com} \\
\addr Microsoft Research, New York
\AND
\Name{Haipeng Luo} \Email{haipeng@microsoft.com}\\
\addr Microsoft Research, New York
\AND
\Name{Behnam Neyshabur} \Email{bneyshabur@ttic.edu}\\
\addr Toyota Technological Institute at Chicago
\AND
\Name{Robert E. Schapire} \Email{schapire@microsoft.com}\\
\addr Microsoft Research, New York
}
\else
\title{Corralling a Band of Bandit Algorithms}  
\author[1]{Alekh Agarwal}
\author[1]{Haipeng Luo}
\author[2]{Behnam Neyshabur}
\author[1]{Robert E. Schapire}
\affil[1]{{\small Microsoft Research, New York}}
\affil[2]{{\small Toyota Technological Institute at Chicago}}
\date{}
\fi

\newenvironment{packed_enum}{
\begin{enumerate}
\setlength{\itemsep}{1pt}
\setlength{\parskip}{0pt}
\setlength{\parsep}{0pt}
}{\end{enumerate}}

\begin{document}

\maketitle

\begin{abstract}
We study the problem of combining multiple bandit algorithms (that
is, online learning algorithms with partial feedback) with the goal of
creating a master algorithm that performs almost as well as the best
base algorithm {\it if it were to be run on its own}.  The main
challenge is that when run with a master, base algorithms unavoidably
receive much less feedback and it is thus critical that the master
not starve a base algorithm that might perform uncompetitively
initially but would eventually outperform others if given enough
feedback.  We address this difficulty by devising a version of
Online Mirror Descent
with a special mirror map together with a sophisticated learning rate
scheme.
We show that this approach
manages to achieve a more delicate balance between
exploiting and exploring base algorithms than previous works
yielding superior regret bounds.

Our results are applicable to many settings, such as
multi-armed bandits, contextual bandits, and convex
bandits.
As examples, we present two main applications.
The first is to create an algorithm that
enjoys worst-case robustness while at the same time performing much
better when the environment is relatively easy.  The second is to
create an algorithm that works simultaneously under different
assumptions of the environment, such as different priors or different
loss structures.
\end{abstract}

\ifcolt
\fi

%%%%%%%%%%%%%%%%%%%%%%%%%%%%%%%%%%%%%%%%%%%%%%%%%%%%%
%%%%%%%%%%%%%%%%%%%%%%%%%%%%%%%%%%%%%%%%%%%%%%%%%%%%%
\section{Introduction}
We study the problem of combining suggestions from a collection of
online learning algorithms in the partial feedback setting, with the
goal of achieving good performance as long as one of these base
algorithms performs well for the problem at hand.

For example, suppose a company wants to do personalized advertising
using some {\it contextual bandit}~\citep{LangfordZh08}
algorithm. Different algorithms in the literature outperform others
under different environments (e.g. i.i.d or adversarial), making it
hard to commit to one of them beforehand. Instead of trying them all
once and committing to the best---an inefficient and nonadaptive
approach---is it possible to come up with an adaptive and automatic
master algorithm whose performance is always competitive with the best
of these base algorithms in any environment?

In the full-information setting where the losses for all actions are
revealed at each round, this problem can be solved simply by running,
for instance, the weighted majority algorithm~\citep{LittlestoneMa89}.
However, this does not directly work in the {\it bandit} setting,
since the base algorithms whose suggestions were ignored cannot update
their internal state. A natural impulse in this case is to run a
multi-armed bandit algorithm (such as EXP3~\citep{AuerCeFrSc02}) as
the master, treating the base algorithms as arms.  By the {regret}
guarantee of a multi-armed bandit algorithm, which states that on
average the performance of the master is almost as good as the best
arm, it seems that our scenario is perfectly addressed.

However, this reasoning is flawed as the base algorithms are not
static arms. While the master algorithm does compete with the base
algorithms in terms of their {\it actual} performance during the run,
this performance could be significantly worse than if the base
algorithm were run on its own, updating its state after every
prediction. For instance, a base algorithm which is exploratory
initially but excels later on might quickly fall out of favor with the
master, effectively meaning that it never gets to explore enough and
reach its good performance regime. Therefore, the real objective of
creating such an ensemble is to make sure that the master performs
almost as well as the best base algorithm if {\it it were to be run on
its own}. As we will see in this paper, this modified objective leads
to an even more delicate explore-exploit trade-off than standard
bandit problems.

%% The main difficulty of this problem is that it requires an even more
%% delicate trade-off between exploration and exploitation than the usual
%% multi-armed bandit problem.  Intuitively, we need to avoid the
%% situation where we starve a base algorithm simply due to its
%% uncompetitive performance at the beginning, since it is possible that
%% it learns slowly initially but would eventually outperform others if
%% given enough feedback.  This essentially requires additional
%% exploration so that each base algorithm always receives a reasonable
%% degree of feedback.  In the extreme we could simply sample a base
%% algorithm uniformly at random and follow its suggestion, which will
%% ensure on average a constant fraction of the total amount of feedback
%% for each base algorithm.  However, it is clear that in this case
%% there is no exploitation at all and the master is bound to fail at
%% keeping up with the best base algorithm.

The most related previous work is by~\citet{MaillardMu11} (see also
the survey of
\citet[Chapter 4.2]{BubeckCe12}) who studied special
cases of our framework.  They essentially run
EXP4~\citep{AuerCeFrSc02} as the master, with some additional uniform
exploration. If the base algorithms are EXP3 or its variants which
have $\order(\sqrt{T})$ regret bounds when run by themselves, where
$T$ is the number of rounds, these works show $\order(T^{2/3})$ regret
for the master---the loss in rates due to the additional uniform
exploration. Whether the regret can be improved to $\order(\sqrt{T})$
in this case was left as a major open problem. \citet{FeigeKoTe14}
model base algorithms as stateful policies and consider a
much harder objective that only admits $\Theta(T/\text{poly}(\ln T))$
results.

In this work we present a generic result for this problem in a much
more general framework, which includes multi-armed, contextual, and
convex bandits, and more (see Section~\ref{sec:setup}), and
affirmatively addresses the setting of the open problem in
particular. In Section~\ref{sec:algo-int}, we first show that in
general no master can have non-trivial regret even when one of the
base algorithms has constant regret, which motivates us to make some
very natural stability assumption on the base algorithms.  With this
assumption, we propose a novel master algorithm, called {\alg}, which
manages to explore more actively but adaptively, and achieve similar
regret bounds as the best base algorithm as shown in
Section~\ref{sec:main}.

Our solution is based on a special instance of the well-studied Online
Mirror Descent framework (see for example~\citep{Shalevshwartz11}),
with a mirror map that in some sense admits the highest possible
amount of exploration while keeping the optimal regret.  This mirror
map was recently studied in~\citep{FosterLiLySrTa16} for a very
different purpose of obtaining first-order regret bounds
and our analysis is also different. Another key
ingredient of our solution is a sophisticated schedule for tuning the
learning rates of the master algorithm, which increases the learning
rate corresponding to a specific base algorithm when it has relatively
low probability of getting feedback. This tuning schedule was also
recently used in~\citep{BubeckElLe16} in a completely different
context of designing computationally efficient convex bandit method.

To show the power of our new approach, in Section~\ref{sec:app} we
present two scenarios where one can directly use our master algorithm
to create a more adaptive solution. The first is to create an
algorithm that guarantees strong robustness in the worst case but at
the same time can perform much better when the environment is
relatively easy (for example, when the data is i.i.d. from a
distribution). The second is to create an algorithm that works
simultaneously under different models (for example, different priors
or different loss structures) and is able to select the correct model
automatically.  Besides algorithms from~\citep{BubeckSl12, SeldinSl14,
AuerCh16} for both stochastic and adversarial multi-armed bandits, our
general results are the first of these kinds to the best of our
knowledge.\footnote{Our regret bounds are always at least $\sqrt{T}$
and do not recover results in~\citep{BubeckSl12, SeldinSl14, AuerCh16}
though.}

We present several examples of these applications in different
settings. For example, going back to contextual bandits example, we
have the following result:

\begin{theorem}[informal]
There is an efficient contextual bandit algorithm with regret $\otil(\sqrt{T})$
if both contexts and losses are i.i.d,
and simultaneously with regret $\otil(T^{3/4})$
if the losses are adversarially chosen, but the contexts are still i.i.d.
\end{theorem}

\ifcolt
\else
Other examples of our results are:

\begin{theorem}[informal]
There is a convex bandit algorithm in $d$ dimensions with regret
$\otil(d^{3/2}\sqrt{T})$ if loss functions are linear,
and simultaneously with regret $\otil((Td)^{5/6})$ if loss functions are only Lipschitz.
\end{theorem}

\begin{theorem}[informal]
There is a $K$-armed stochastic bandit algorithm with (Bayesian) regret $\otil(\sqrt{MTK}H(\theta^*))$
as long as one of the $M$ prior distributions is true,
where $H(\theta^*)$ is the entropy of the prior distribution of the optimal arm.
\end{theorem}
\fi

%The structure of the rest of the paper is as follows.
%Section~\ref{sec:setup} presents the formal setup of our framework
%with various examples.
%Section~\ref{sec:algo-int} starts with a hardness result of our problem
%without making any assumptions,
%then moves on to a natural condition on the base algorithms
%and our new algorithm. 
%The regret guarantees of our algorithm are
%presented in Section~\ref{sec:main}, and their various applications are
%presented in Section~\ref{sec:app}.  We conclude the paper with two
%major open problems in Section~\ref{sec:conclusion}.
%All proofs are deferred to the appendix.

\section{Formal Setup}\label{sec:setup}

We consider a general online optimization problem with bandit
feedback, which can be seen as a repeated game between the environment
and the learner.  On each round $t = 1, \ldots, T$: 
\begin{packed_enum}
\item the environment first reveals some {\em side information} $x_t \in
  \calX$ to the learner; 
\item the learner makes a {\em decision} $\theta_t \in \Theta$ for some
  decision space $\Theta$, while simultaneously the environment decides
  a {\em loss function} $f_t: \Theta \times \calX \rightarrow [0,1]$; 
\item finally, the learner incurs and observes (only) the loss
    $f_t(\theta_t, x_t)$.  
\end{packed_enum}

For simplicity, we measure the performance of an algorithm by its
{\em (pseudo-)regret}, defined as
\[
\sup_{\theta \in \Theta}  \E\left[ \sum_{t=1}^T f_t(\theta_t, x_t) -
  f_t(\theta, x_t) \right] 
\]
where the expectation is taken over the randomness of both the player
and the environment.\footnote{For conciseness, in the rest of the paper we
simply call this the regret.}

Throughout the paper we will talk about different environments.
Formally, an environment $\calE$ is a randomized mapping from
the history $(x_s, \theta_s, f_s)_{s=1, \ldots, t-1}$ to a new outcome
$(x_t, f_t)$.  Equivalently, one can also assume that all randomness
of the environment is drawn ahead of time, which allows us to capture
Bayesian settings too. We use the notation $(x_t, f_t) = 
\calE(\theta_1, \ldots, \theta_{t-1})$ to make
the dependence on the learner's decisions explicit.

This general setup subsumes many bandit problems studied in the
literature including multi-armed, convex, and contextual bandits.  At
a high-level, the decision sets correspond to policies or action sets
available to the player, and environments capture assumptions on the
adversary such as being oblivious or stochastic. We present an
example that instantiates all these quantities concretely at the end of this section, with more
detailed examples in Appendix~\ref{app:examples}.

We assume that we are given a set of $M$ bandit algorithms, denoted by
$\base{1}, \ldots, \base{M}$, each designed for the general setup
above for some environments and decision space
$\Theta_i \subset \Theta$. We refer to these as \emph{base
algorithms}. We aim to develop a \emph{master algorithm} which makes a
decision on each round after receiving suggestions from the base
algorithms. We restrict the master to pick amongst the suggestions of
the base algorithms so that it does not need to know any details of
the base algorithms or the problem itself.

Our goal is to ensure that the performance of the master is not far
away from the best base algorithm had it been run separately.  As
discussed earlier, this is challenging since each base algorithm has
access to a much smaller amount of data when run with a master than on
its own; nevertheless, we want to compete with the counterfactual in
which a single base algorithm drives all of the decisions and receives
feedback on every round. We capture the behavior of a base algorithm
$\base{i}$ using its promised regret bound when run in
isolation. Suppose for some (randomized) environment, $\base{i}$
produces a sequence $\theta^i_1, \ldots, \theta^i_T \in \Theta_i$,
such that the following bound holds:\footnote{In general, regret should
also depend on other parameters such as the size of the decision space
$\Theta$, but we will treat these parameters as fixed constants and
see the regret as solely a function of $T$ when losses are in $[0,1]$.
}
\[
\sup_{\theta \in \Theta_i}  \E\left[ \sum_{t=1}^T f_t(\theta^i_t, x_t) -
  f_t(\theta, x_t) \right] \leq \reg_i(T),
\]
for some regret bound $\reg_i: \N_+ \rightarrow \R_+$. Then, ideally,
we might hope that under the same environment, if we run the master
with all these base algorithms to make the decisions
$\theta_1, \ldots, \theta_T$, we have
\begin{equation}
\sup_{\theta \in \Theta_i}  \E\left[ \sum_{t=1}^T f_t(\theta_t, x_t) -
  f_t(\theta, x_t) \right] \leq \order(\text{poly}(M)\reg_i(T)), 
\label{eqn:desired}
\end{equation}
where the expectation is taken over the randomness of the master, the
base algorithms and the environment. In words, we want the expected
loss of the master to be competitive with the expected loss of the
best decision $\theta$ in the decision space of each base algorithm
$\base{i}$, up to a level which depends on the regret of $\base{i}$.
This problem is beguilingly subtle.  And in this ideal, aspirational
form, there is reason to doubt that such a master algorithm can exist
at all in general.  Nevertheless, in this paper, we make significant
progress toward developing such an algorithm.  But to be clear, the
results are subject to important caveats and conditions that we state
precisely in Section~\ref{sec:algo-int}.

\begin{example}[Contextual bandits]
\label{ex:cb}
\textup{In contextual bandits~\citep{LangfordZh08}, the side
information $x_t$ is typically called a context, the decision space
$\Theta$ is a set of policies $\theta~:~X\rightarrow [K]$ and the loss
function takes the form $f_t(\theta, x)
= \inner{\bc_t}{\basis_{\theta(x)}}$ for some $\bc_t \in [0,1]^K$
specifying the loss of each action at round $t$.
(Here and throughout the paper,
$[n]$ denotes the set $\{1, \ldots, n\}$,
and $\basis_i$ denotes the $i$-th standard basis vector.)
%(whose $i$-th coordinate is $1$ and whose dimension will be clear based on the context).
This problem has been
studied under three main environments:}

\textbf{Stochastic contexts and losses:} 
\textup{ Here the environment is characterized by a fixed distribution
from which contexts $x_t$ and losses $\bc_t$ are drawn i.i.d. The
Epoch-Greedy algorithm of~\citet{LangfordZh08} suffers an expected
regret of $\otil(T^{2/3})$ in this setting, while
\citet{AgarwalHsKaLaLiSc14} get the optimal
$\otil(\sqrt{T})$ regret at a higher computational cost.  }

\textbf{Adversarial contexts or losses:} 
\textup{Several authors~\citep{auer2002using, ChLiReSc11, FilippiCaGaSz10} have
studied environments where the contexts are chosen by an adversary,
but the losses come from a fixed conditional distribution given
$x_t$. Other authors~\citep{SyrgkanisLuKrSc16, rakhlin2016bistro} have
considered contexts drawn i.i.d. from a fixed distribution, but the
losses picked in an adversarial manner. \citet{SyrgkanisLuKrSc16} have
proposed a computationally efficient algorithm which suffers an
expected regret of at most $\otil(T^{2/3})$ in this setting.  }

\textbf{Adversarial contexts and losses:} 
\textup{The EXP4 algorithm of~\citet{AuerCeFrSc02} incurs an expected
regret at most $\otil(\sqrt{T})$ in this most general setting, but is
computationally inefficient.  } 
\end{example}

\section{Assumption and Algorithm}
\label{sec:algo-int}

Intuitively, the task of the master algorithm appears quite similar to
a standard multi-armed bandit problem, with each base algorithm as an
arm. However, as hinted in the introduction, this is not the case --- the
problem admits no non-trivial results without further assumptions.
In this vein, we
now present a lower bound, and an assumption to avoid it. After
understanding the failure of typical algorithms despite making the
assumption, we then present our approach.

\subsection{Hardness in the Worst Case and A Natural Assumption}
\label{sec:hardness}

We begin with the following hardness result (see Appendix~\ref{app:lower_bound} for the proof).

\begin{theorem}\label{thm:lower_bound}
There is an environment and a pair of base algorithms $\base{1},
\base{2}$ such that either $\reg_1(T)$ or $\reg_2(T)$ is a constant
(independent of $T$), but for any master algorithm combining
$\base{1}, \base{2}$, the expected regret of the master is at least
$\Omega(T)$ (thus, the bound~\eqref{eqn:desired} does not hold).
\label{thm:lb}
\end{theorem}

This lower bound and its proof highlight the main challenge of our problem.
The assumption of a good regret bound on
each base algorithm when run in isolation in an environment is not
sufficient since we can come up with pathological examples where
the behavior of the algorithm completely changes when run under a
master. Therefore, we next consider
natural modifications of the environment of a base algorithm, to which
we expect robustness. 

%We have observed before that the problem faced by the master is akin
%to a multi-armed bandit problem. It is, therefore, instructive to
%consider how typical multi-armed bandit algorithms work in order to
%develop better intuition about possible solutions. As described in
%Example~\ref{ex:mab}, in a multi-armed bandit problem, an algorithm
%plays an action $i_t \in [K]$ drawn according to a distribution $\bp_t
%\in \Delta_K$ at time $t$, and observes $c_{t, i_t}$. In order to update
%$\bp_t$ to $\bp_{t+1}$, it is convenient to have access to the losses of
%all the actions, but only the loss on $i_t$ is observed. This is
%typically circumvented by creating an \emph{importance-weighted loss
%  estimate} $\hat{\bc}_t$ where 
%\begin{equation}
%\hat{c}_{t,i} = c_{t,i} \frac{\one\{i=i_t\}}{p_{t,i}}~.
%\label{eqn:imp_wt}
%\end{equation}
%%
%It is clear that if $i_t \sim \bp_t$, then $\hat{\bc}_t$ is an
%unbiased estimator for the entire loss vector $\bc_t$. This loss
%estimate is then fed to a \emph{full-information} online learning
%algorithm, one which has access to the entire loss vector, and is used
%to obtain $\bp_{t+1}$ from $\bp_t$. 

Recall that in our setting, the master only observes the loss for the decision
suggested by one of the base algorithms it picked (randomly). 
It is thus natural to create importance-weighted
losses for each base algorithm. %, such as: $\loss_{t,i} = f_t(\theta_t,
%x_t)\one\{i = i_t\}/p_{t,i}$ where $p_{t,i}$ is the probability of the
%master picking $\base{i}$ at round $t$. 
%

To this end, for an environment $\calE$, we define the {\em
environment $\calE'$ induced by importance weighting}, which is the
environment that results when importance weighting is applied to the
losses provided by environment $\calE$.
More precisely, $\calE'$ is defined as follows.
On each round $t = 1, \ldots, T$,
\begin{packed_enum}
\item $\calE'$ picks an arbitrary sampling probability $p_t \in [0,1]$ and obtains 
$(x_t, f_t) = \calE(\theta_1', \ldots, \theta_{t-1}')$.

\item $\calE'$ reveals $x_t$ to the learner and the learner makes a decision $\theta_t$.

\item With probability $p_t$, define $f_t'(\theta, x) = f_t(\theta, x)/p_t$ and $\theta_t' = \theta_t$;
with probability $1-p_t$, define $f_t'(\theta, x) \equiv 0$ and $\theta_t' \in \Theta$ to be arbitrary. 

\item $\calE'$ reveals the loss $f_t'(\theta_t, x_t)$ to the learner, and passes $\theta_t'$ to $\calE$.
\end{packed_enum}

Such an induced environment is exactly the one that the base
algorithms face when run with a master using importance-weighted
losses.  If the base algorithms have similar performance under $\calE$
and $\calE'$, then we can exclude the pathological examples which
govern our lower bound.  However, while the original loss $f_t$ is in
$[0,1]$, the estimated loss $f_t'$, although an unbiased estimate of
$f_t$, takes values in the larger range $[0, 1/p_t]$, meaning that the
range of losses has changed significantly.  We therefore define the
following notion of stability of the base algorithms, which captures
how much an algorithm's regret degrades as a result of the range
expanding in this fashion:

\begin{definition}\label{def:stable}
For some $\alpha \in(0, 1]$ and non-decreasing function
$\calR: \N_+ \rightarrow \R_+$, an algorithm with decision space
$\Theta_0 \subset \Theta$ is called $(\alpha, \calR)$-stable with
respect to an environment $\calE$ if its regret under $\calE$ is
$\calR(T)$, and its regret under any environment $\calE'$ induced by
importance weighting is
\begin{equation}
  \sup_{\theta \in \Theta_0}~\E\left[\sum_{t=1}^T f_t'(\theta_t, x_t) -
      f_t'(\theta, x_t) \right] \leq \E\left[\scale^{\alpha}\right]\reg(T)
\label{eqn:stability}
\end{equation}
where $\scale = \max_{t \in [T]} 1/p_t$ (with $p_t$ as in the
definition of $\calE'$ above), and all expectations are 
taken over the randomness of both $\calE'$ and the algorithm.
\end{definition}

%Also note that if $1/p_{t,i} \leq \scale$ for
%all $i \in [M]$, $t \in [T]$ and some $\scale > 0$, then the second 
%moment of the loss is bounded by $\scale$:
%$ \E_{i_t \sim \bp_t} [{\loss}_{t,i}^2] \leq \scale,~\forall i. $,
%an important departure from worst-case behavior $\scale^2$.
%
%To allow the master algorithm to utilize techniques like
%importance-weighted loss estimation where there exists a
%quantity like $\scale$ above bounding both the
%magnitude and the second moment of the losses fed to the base algorithms,
%we assume that each base algorithm $\base{i}$, when
%provided with such a loss sequence, 
%guarantees the following regret bound:
%%
%\begin{equation}
%  \sup_{\theta \in \Theta_i}~\E\left[\sum_{t=1}^T f_t(\theta_t, x_t) -
%      f_t(\theta, x_t) \right] \leq \E[\scale^{\alpha_i}]\reg_i(T),
%\label{eqn:assume}
%\end{equation}
%%
%for some function $\reg_i: \N_+ \rightarrow \R_+$ and exponent $0
%< \alpha_i \leq 1$.  

This stability assumption intuitively posits that the regret
of the algorithm grows at most linearly in the scale of the losses it receives.  
In the adversarial construction of
Theorem~\ref{thm:lb} (in Appendix~\ref{app:lower_bound}), we can see
that this is certainly not the case there.  However, for most
``reasonable'' base algorithms, a linear scaling with $\alpha = 1$
is trivially achievable simply by rescaling the losses.
%We are running
%an algorithm with losses magnified by at most a factor of $\scale$. It
%is natural that the regret should grow by the same multiplicative
%factor. 
Moreover, note that the second moment of the loss estimate $f_t'$ is also bounded
by $\scale$ (instead of $\scale^2$): $\E_{p_t} [f_t'(\theta_t, x_t)^2] \leq \scale$,
and as we will see in the sequel, the regret of many natural bandit
algorithms does scale as some function of the second moment of
the loss sequence. In such cases, it is typical to obtain an exponent
$\alpha$ strictly smaller than $1$. 

There are two seemingly strong parts about this condition.
First, the bound requires adaptation to the quantity $\scale$
which is unknown to the algorithm ahead of time.  However, this
can be easily resolved by a standard doubling
trick~\citep{CesabianchiFrHeHaScWa97}. 
Second, if an algorithm is designed for an
i.i.d. environment $\calE$, then we might not expect to have any regret guarantee 
under $\calE'$ since it is not an i.i.d. environment anymore.
However, even in this case, due to the special structure of $\calE'$,
one can still prove stability for many i.i.d. algorithms as we will show later.

In conclusion, our stability condition is a natural and mild requirement for an algorithm.
In Appendix~\ref{appd:app}, we show how this condition is satisfied for
most existing bandit algorithms, either as is, or by extremely simple
modifications (also see Table~\ref{tab:examples} for a summary).

Armed with the assumption, it is natural to revisit a question from
before: can we use any existing multi-armed bandit algorithm as a
master and hope to get guarantee~\eqref{eqn:desired} under this
assumption?  It turns out that the answer is still no if we were to
use an arbitrary multi-armed bandit algorithm as a master. To see why,
consider the classic multi-armed bandit algorithm
EXP3~\citep{AuerCeFrSc02} as the master. EXP3 induces probabilities
that are exponentially small in the cumulative loss of $\base{i}$,
meaning that the scaling $\scale$ can grow exponentially large with
$T$. We can mitigate this problem partially by adding additional
uniform exploration to EXP3, but one can verify that such
modifications unavoidably lead to a major deterioration
in the regret (for example
$\order(T^{2/3})$ regret of the master even when all the base
algorithms have $\order(\sqrt{T})$ regret). This is exactly the issue
noted in prior works~\citep{MaillardMu11, BubeckCe12}, as mentioned in
the introduction.

In the next subsection, we present a specific multi-armed bandit
algorithm that does address all these issues successfully, and
provide results on its performance in the sections that follow.

\subsection{Our Algorithm}
\label{sec:algo}

\begin{algorithm2e}[t]
\DontPrintSemicolon
\caption{\alg}
\label{alg:main}
\KwIn{learning rate $\eta$ and $M$ base algorithms $\base{1}, \ldots, \base{M}$} 
Initialize: $\gamma = 1/T$, $\beta = e^{\frac{1}{\ln T}}$, $\eta_{1,i} = \eta$, $\scale_{1,i} = 2M$ for all $i \in [M]$, $\bp_1 = \bpbar_1 = \frac{\one}{M}$ \\
Initialize all base algorithms

\For{$t=1$ \KwTo $T$} {
    Observe side information $x_t$, send $x_t$ to $\base{i}$ and receive decision $\theta_t^{i}$ for each $i \in [M]$ \\
    Sample $i_t \sim \bpbar_t$, predict $\theta_t = \theta_t^{i_t}$, observe loss $f_t(\theta_t, x_t)$   \label{line:sampling}  \\
    Send $f_t^i(\theta_t^i, x_t)$ to $\base{i}$ as feedback for each $i \in [M]$ where $f_t^i(\theta, x) = \frac{f_t(\theta, x)}{\pbar_{t, i_t}}  \one\{i = i_t\}$ \label{line:feedback} \\
    Update $\bp_{t+1} = \LOMD(\bp_t, \frac{f_t(\theta_t, x_t)}{\pbar_{t,i_t}} \basis_{i_t}, \boldeta_t)$  \label{line:LOMD} \\
    Set $\bpbar_{t+1} = (1 - \gamma) \bp_{t+1} + \gamma \frac{\one}{M}$ \label{line:mixing} \\
    \For{$i=1$ \KwTo $M$} {
        \lIf{$\frac{1}{\pbar_{t+1, i}} > \scale_{t, i}$} {set $\scale_{t+1, i} = \frac{2}{\pbar_{t+1, i}}$, $\eta_{t+1, i} = \beta\eta_{t,i}$   \label{line:doubling}}
        \lElse{set $\scale_{t+1, i} = \scale_{t,i}$, $\eta_{t+1, i} = \eta_{t,i}$  }
    }
}
\end{algorithm2e}

\begin{algorithm2e}[t]
\DontPrintSemicolon
\caption{$\LOMD(\bp_t, \bloss_t, \boldeta_t)$}
\label{alg:LOMD}
\KwIn{previous distribution $\bp_t$, current loss $\bloss_t$ and learning rate vector $\boldeta_t$}
\KwOut{updated distribution $\bp_{t+1}$}
Find $\lambda \in [\min_i \loss_{t,i}, \; \max_i \loss_{t, i}] $ such that $\sum_{i=1}^M \frac{1}{\frac{1}{p_{t,i}} + \eta_{t,i} (\loss_{t,i} - \lambda)} = 1$ \\
Return $\bp_{t+1}$ such that $\frac{1}{p_{t+1,i}} = \frac{1}{p_{t,i}} + \eta_{t,i} ( \loss_{t,i} - \lambda)$ 
\end{algorithm2e}

It is well known that EXP3 belongs to a large family of algorithms
called {\it Online Mirror Descent}, and it is thus natural to ask
whether there is a different instance of Online Mirror Descent that
solves our problem.  Specifically, let $\bp_t$ be the distribution
for picking a base algorithm on round $t$, and $\bloss_t$ be some loss
estimator of the base algorithms.  Online Mirror Descent updates
$\bp_t$ as follows:
\begin{equation*}
\begin{split}
\nabla \psi_t(\bptilde_{t+1}) &= \nabla \psi_t(\bp_t) - \bloss_t \\
\bp_{t+1} &= \argmin{\bp \in \Delta_M} \bregman{\psi_t}{\bp}{\bptilde_{t+1}}
\end{split}
\end{equation*}
where $\psi_1, \ldots, \psi_T$ are the {\em mirror maps} and
$\bregman{\psi}{\bp}{\bq} = \psi(\bp) - \psi(\bq) -
\inner{\nabla\psi(\bq)}{\bp - \bq}$ is the {\em Bregman divergence}
associated with $\psi$. For example, EXP3 with a learning rate $\eta$
uses negative entropy $\psi_t(\bp) =
\frac{1}{\eta}\sum_{i=1}^M p_i \ln p_i$ as the mirror map and
updates $p_{t+1,i} \propto \exp(-\eta \sum_{s=1}^t
\loss_{s,i})$.  As discussed before, due to the exponential weighting,
$p_{t,i}$ can be very small and thus more likely lead to the
starvation of the base algorithms that perform poorly initially.
Various other mirror maps have been proposed in the literature and
might be considered for our purpose, most
notably the negative Tsallis entropy~\citep{AudibertBu10,
AbernethyLeTe15}.

We would suggest, however, that the most suitable mirror map is 
 $\psi_t(\bp) = - \frac{1}{\eta}\sum_{i=1}^M \ln p_i$, originally
suggested in a recent work of~\citet{FosterLiLySrTa16}.  This choice
is motivated by our previous discussion which suggests that a good
master algorithm should ensure that $1/p_{t,i}$ does not grow too
rapidly. Typically, this quantity grows exponentially in 
 $\eta\sum_{s=1}^t \loss_{s,i}$ for EXP3, and
polynomially for Tsallis entropy with a degree greater than $1$. On the other
hand, the mirror map suggested above gives the simple update
$p_{t+1,i} = (\eta\sum_{s=1}^t \loss_{s,i} + Z)^{-1}$ where $Z$ is a
normalization factor, which means that $1/p_{t,i}$ grows just linearly
in  $\eta\sum_{s=1}^t \loss_{s,i}$ and appears to be the least extreme
weighting amongst all known algorithms that yield $\sqrt{T}$ regret in
the bandit setting. Thus, this mirror map is likely to provide a high
level of exploration without hurting the regret guarantees.  This
special instance of Online Mirror Descent is a key component of our
algorithm.  Since this mirror map resembles the log barrier for the
positive orthant~\citep{NesterovNe94}, we call it
{\LOMD}.\footnote{Note that this is different from directly using a
barrier {\it for the simplex} as proposed in~\citep{AbernethyHaRa12,
RakhlinSr13}. } Note that while {\LOMD} was proposed
in~\citep{FosterLiLySrTa16} to provide a so-called ``small-loss''
regret bound, here we are not using this special regret bound but a
rather different property of the algorithm.

OMD methods typically use a decaying learning rate schedule to ensure
convergence. Intuitively, this is at odds with our learning goal. If a
base algorithm $\base{i}$ starts to do well after underperforming for
a while, we want to exploit this good performance. We achieve this by
instead considering \emph{non-decreasing} learning rate
schemes. However, a schedule merely based on the round $t$ is not
sufficient since we would like to adjust the learning rate for
$\base{i}$ based on its performance. Hence, we allow each $\base{i}$
to have its own learning rate $\eta_{t,i}$, which corresponds to using
$\psi_t(\bp) = -\sum_{i=1}^M \frac{\ln p_i}{\eta_{t,i}}$ as the actual mirror map.
The precise setting of $\eta_{t,i}$ is discussed further below.

Our main algorithm, called {\alg}, is presented in
Algorithm~\ref{alg:main}.  It is clear that
Lines~\ref{line:sampling}, \ref{line:feedback} and~\ref{line:LOMD} are
essentially performing {\LOMD} over the base algorithms with the usual
unbiased loss estimator $\bloss_t = \frac{f_t(\theta_t,
x_t)}{\pbar_{t,i_t}} \basis_{i_t}$.  Note that we sample the base
algorithms according to $\bpbar_t$, a smoothed version of $\bp_t$ (see
Line~\ref{line:mixing}, where $\one$ denotes the all-ones vector),
which can be seen as adding another hard constraint to prevent the
weighting from becoming too extreme.  Moreover, one can verify that
{\LOMD} admits a simple update formula as presented in
Algorithm~\ref{alg:LOMD}.  Indeed, plugging the gradients of the log
barrier mirror map gives $\frac{1}{\ptilde_{t+1,i}}
= \frac{1}{p_{t,i}} + \eta_{t,i}
\loss_{t,i}$.  On the other hand, the Bregman divergence is
$\bregman{\psi_t}{\bp}{\bq} = \sum_{i=1}^M
\frac{1}{\eta_{t,i}}h\left(\frac{p_i}{q_i}\right)$ with $h(y) =y - 1-
\ln y$, and by standard analysis the projection step can be solved as
in Algorithm~\ref{alg:LOMD} by finding the Lagrange multiplier
$\lambda$ via a line search.

Line~\ref{line:doubling} gives our setting of the learning rates
$\eta_{t,i}$. This setting is motivated by our analysis, where we show
that one of the terms for the regret of \alg against $\base{i}$ scales roughly as 
$
\sum_{t=1}^T\frac{1}{\pbar_{t+1,i}}\left(\frac{1}{\eta_{t+1,i}}
- \frac{1}{\eta_{t,i}} \right). 
$
Since $\eta_{t,i}$ is non-decreasing in $t$, this term is always
nonpositive and its magnitude is large when $\pbar_{t+1,i}$ is small. This conveys the
intuition that when \alg puts low probability on $\base{i}$, it is
incurring negative regret to $\base{i}$. While several settings of
$\eta_{t,i}$ are plausible, we find it convenient to analyze the
evolution of this term if $\eta_{t,i}$ is updated as in Line~\ref{line:doubling}
where it is increased by a factor of $\beta$ whenever $1/\pbar_{t+1,i}$
is larger than some threshold.
On each such event, we know that
$1/\pbar_{t+1,i}$ is large so that we gain a large negative term in the
regret. At the same time, $\eta_{T,i}$ is at most a constant factor
larger than $\eta_{1,i}$ which is crucial in other parts of our
analysis.  For ease of bookkeeping, we also maintain the threshold $\scale_{t,i}$ 
which is always an
upper bound on $1/\pbar_{t,i}$ (and hence the loss that $\base{i}$ receives on round $t$), 
and update it in a doubling manner. These thresholds essentially
correspond to the quantity $\scale$ in the stability condition.
We note that similar non-decreasing learning rate schedule were originally proposed
in~\citep{BubeckElLe16} for a different problem.

\begin{remark*}
An alternative for the feedback to the base algorithms and the loss estimator is to let 
\[
f_t^i(\theta, x) =
  \frac{f_t(\theta, x)}{\sum_{j:\theta^j_t = \theta^{i_t}_t}
    \pbar_{t, j}} \one\{\theta^i_t = \theta^{i_t}_t\}
    \quad\text{and} \quad
  \loss_{t,i} = f_t^i(\theta_t^i, x_t)
\] 
  which is less
  wasteful when $\Theta$ is a small set and it is more likely that
  multiple base algorithms make the same decision. One can verify that
  this alternative provides the same guarantee as Algorithm~\ref{alg:main}. 
\end{remark*}

\section{Main Results}\label{sec:main}
In this section, we present our main results which give
guarantees on the performance of \alg\ when base algorithms
satisfy the stability condition defined in Definition~\ref{def:stable}. We first give a
general result before elaborating on its various implications. 

\begin{theorem}
For any $i \in [M]$, if base algorithm $\base{i}$ (with decision space
$\Theta_i$) is $(\alpha_i, \calR_i)$-stable (recall
Defn.~\ref{def:stable}) with respect to an environment $\calE$, then
under the same environment \alg satisfies
%Given an environment and an $i \in [M]$, suppose a base algorithm
%$\base{i}$ satisfies 
%\begin{equation}\label{eq:cond}
%  \sup_{\theta \in \Theta_i} \E\left[ \sum_{t=1}^T f_t^i(\theta_t^i, x_t)
%    - f_t^i(\theta, x_t) \right] \leq \E\left[\scale_{T,i}^{\alpha_i}\right] \reg_i(T)
%\end{equation}
%for some function $\reg_i:~\N_+\rightarrow \R_+$
%and some constant $0 < \alpha_i \leq 1$.
%Then \alg satisfies
\begin{equation}
  \sup_{\theta \in \Theta_i} \E\left[ \sum_{t=1}^T f_t(\theta_t, x_t)
    - f_t(\theta, x_t) \right] = \otil\left(\frac{M}{\eta} + T\eta -
  \frac{\E[\scale_{T,i}]}{\eta} +
  \E[\scale_{T,i}^{\alpha_i}]\reg_i(T)\right),
  \label{eqn:main-meta}
\end{equation}
where all expectations are taken over the randomness of
Algorithm~\ref{alg:main}, the base algorithms, and the environment.
\label{thm:main-meta}
\end{theorem}

%Recalling the discussion in Section~\ref{sec:hardness}, the parameter
%$\scale$ in Eq.~\eqref{eqn:assume} is clearly bounded by
%$\scale_{T,i}$ in \alg for base algorithm $\base{i}$ since
%$\scale_{t,i}$ is non-decreasing in $t$ and the $f_t$'s take values in
%$[0,1]$. According to earlier discussions, Eq.~\eqref{eq:cond} is thus
%a natural condition. 

%A subtle point about Eq.~\eqref{eq:cond} is that while the left-hand
%side appears
%to be the usual definition of the regret of the base algorithm, the
%loss sequence $f_t^i$ actually depends on not only the behavior of
%$\base{i}$, but also the behavior of the master and other base algorithms
%(hence the expectation is also taken over all the randomness). In some
%cases this means that the base algorithm is facing a more powerful
%environment. For example, if an algorithm is designed for an
%i.i.d. environment, Eq.~\eqref{eq:cond} is requiring that the same
%regret bound (up to the factor $\E[\scale_{T,i}^{\alpha_i}]$) still
%holds under a very specific environment where losses are not
%i.i.d. anymore. However, we will prove in Appendix~\ref{appd:app} that
%many algorithms indeed satisfy this condition.

%Note that by construction, for each $i$ we have $\max_t 1/\pbar_{t,i} \leq \scale_{T,i}$,
%and thus $\scale_{T,i}$ is an upper bound on the quantity $\rho$ in
%Definition~\ref{def:stable}.

In our setup, the master algorithm \alg ends up playing the role of
the importance-weighting adversary in Definition~\ref{def:stable} and
the theorem assumes stability of the base algorithms to such
modifications of the loss. The regret bound given in
Eq.~\eqref{eqn:main-meta} can be interpreted intuitively as
follows. The first two terms arise as the typical regret incurred in a
standard adversarial multi-armed bandit problem. The next two terms
capture the distinct aspects of our hierarchical setup. The last term
comes directly from the regret of $\base{i}$ relative to $\Theta_i$
according to the stability condition.  The negative term arises as
discussed in Section~\ref{sec:algo}. \alg induces low probabilities,
and hence large $\scale_{T,i}$ on $\base{i}$ if it finds that
$\base{i}$ is being consistently outperformed by some other base
algorithm. In this case, the master gets a negative regret with
respect to $\base{i}$ by increasing the learning rate for $\base{i}$.
Note that this term scales with $\scale_{T,i}$, which is crucial in
obtaining better regret than prior works.

%However, the theorem consists of several terms, and it is hard to
%understand the precise rates it yields, the desired setting for $\eta$ etc. 
In order to better understand this general result, we
now further simplify the theorem for two special cases. We start with
the case where the master wants to guarantee low regret against a base
algorithm $\base{i}$ with $\alpha_i = 1$. 
That is, the regret of the base
algorithm scales linearly with the range of the losses. In this
setting, we obtain the following result. 

\begin{theorem}\label{thm:main1}
Under the conditions of Theorem~\ref{thm:main-meta}, if
$\alpha_i = 1$, then with $\eta
= \min\left\{\frac{1}{40\reg_i(T)\ln T},
\sqrt{\frac{M}{T}}\right\}$ \alg satisfies: 
$
\sup_{\theta \in \Theta_i} \E\left[ \sum_{t=1}^T f_t(\theta_t, x_t) -
  f_t(\theta, x_t) \right] = \otil\left(\sqrt{MT} + M\reg_i(T)\right).
$
\end{theorem}

Theorem~\ref{thm:main1} is evidently yielding the ideal
bound~\eqref{eqn:desired}, but with an important caveat. The initial
learning rate $\eta$ needs to be set based on the regret bound
$\reg_i(T)$ of the base algorithm $\base{i}$ that we wish to compete
with. One way of interpreting this theorem is the following. Suppose
for a given environment, $S$ is the set of all base algorithms which
we are interested in competing with and which
satisfy the condition with $\alpha_i=1$. 
Let $\reg_{\max}(T) = \max_{i\in
  S}\reg_i(T)$. Then we get: %% can instead set $\eta$ in terms of $\reg_{\max}(T)$
%% and obtain regret $\otil(\sqrt{MT} + M\reg_{\max}(T))$ for all $i \in S$.

\begin{corollary}
 Under the conditions of Theorem~\ref{thm:main1}, with $\eta
 = \min\left\{\frac{1}{40\reg_{\max}(T)\ln
 T}, \sqrt{\frac{M}{T}}\right\}$, \alg satisfies:
 $ \sup_{\theta \in \Theta_i} \E\left[ \sum_{t=1}^T f_t(\theta_t, x_t)
 - f_t(\theta, x_t) \right] = \otil\left(\sqrt{MT} +
 M\reg_{\max}(T)\right), \quad \mbox{for all $i \in S$}.
 $ \label{cor:main1}
\end{corollary}

The main application of this result is the following situation that
resembles model selection problems. One has a collection of base
algorithms such that each of them works well under a different
environment or decision space, and we want one single robust algorithm
that works well simultaneously across all these environments and
decision spaces (see Section~\ref{sec:models} for examples).

However, there is also another related class of applications for which
Theorem~\ref{thm:main1} turns out not to be so helpful. For instance,
in contextual bandits, we might consider using for base algorithms:
(1)~the algorithm of \citet{AgarwalHsKaLaLiSc14}, which is nearly
statistically optimal but computationally somewhat expensive, using a
fairly small policy class; and (2)~the Epoch-Greedy algorithm
of \citet{LangfordZh08}, which is statistically suboptimal but
computationally cheap, using a larger policy class.  If we apply
Theorem~\ref{thm:main1} with these base algorithms, either we suffer a
substantially suboptimal regret of $\otil(T^{2/3})$ against both
policy classes, or we do not end up with any non-trivial guarantee
against the richer class used by Epoch-Greedy. Our next theorem
partially alleviates such concerns by exploiting the case when
$\alpha_i < 1$ (in fact all our examples admit $\alpha_i < 1$).

%The key observation is that the
%regret of the base algorithms when run with {\alg} does not
%necessarily scale linearly in the loss range as it should have in the
%worst case.  The reason is that the base algorithms are in fact
%receiving random losses with variance at most $\scale_{T,i}$ instead of
%the worst case $\scale_{T,i}^2$.  When better guarantees of the base
%algorithms exist due to this fact, we have the following improvement.

\begin{theorem}\label{thm:main2}
Under the conditions of Theorem~\ref{thm:main-meta}, if $\alpha_i < 1$ then \alg satisfies\\
\[
\sup_{\theta \in \Theta_i} \E\left[ \sum_{t=1}^T f_t(\theta_t, x_t) -
  f_t(\theta, x_t) \right] = \otil\left(\frac{M}{\eta} + T\eta +
\reg_i(T)^{\frac{1}{1-\alpha_i}}
\eta^{\frac{\alpha_i}{1-\alpha_i}}\right).
\]
\end{theorem}

%We will show that most bandit algorithms satisfy $\alpha_i < 1$
%naturally in Section~\ref{sec:app}.  
To see why Theorem~\ref{thm:main2} is useful, consider again the
contextual bandit scenario discussed above and assume $\alpha_i = 1/3$
for Epoch-Greedy (we will indeed prove this in
Appendix~\ref{appd:app}).  In this case, choosing $\eta =
\Theta(\frac{1}{\sqrt{T}})$ ensures 
$\order(\sqrt{T})$ regret against the small policy class while
 slightly larger $\order(T^{3/4})$ regret against the richer
policy class --- a guarantee that neither of the base algorithms
provides by itself.

%The results of Theorems~\ref{thm:main1} and~\ref{thm:main2} are rather
%abstract, but they can be instantiated and applied in a variety of
%interesting settings. In the following section, we
%provide some examples of practically useful conclusions which are
%obtained as direct corollaries of our results. 

\section{Applications}\label{sec:app}

We present several concrete examples of using {\alg} in this section.
The base algorithms we consider for different problems are listed in Table~\ref{tab:examples}
along with their stability parameters under a specific class of environments.
All details and proofs can be found in Appendix~\ref{appd:app}.

\renewcommand{\arraystretch}{0.5}
\begin{table*}[t]
\vskip 0.15in
\centering
\begin{tabular}{|c|c|c|c|}
\hline
Algorithm & $\reg(T)$ & $\alpha$ & Environment \\
\hline
ILOVETOCONBANDITS & $\otil(\sqrt{KT\ln|\Theta|})$ & $1/2$ & stochastic contextual bandit \\
%\hline
BISTRO+&  $\order((KT)^{\frac{2}{3}}(\ln|\Theta|)^{\frac{1}{3}})$ & $1/3$ & hybrid contextual bandit \\
%\hline
Epoch-Greedy & $\otil(T^{\frac{2}{3}} \sqrt{K\ln|\Theta|})$ & $1/3$ & stochastic contextual bandit \\
%\hline
EXP4 & $\order(\sqrt{KT\ln|\Theta|})$ & $1/2$ & adversarial contextual bandit \\
%\hline
SCRiBLe & $\otil(d^{\frac{3}{2}}\sqrt{T})$ & $1/2$ & adversarial linear bandit \\
%\hline
BGD & $\order(d\sqrt{L}T^{\frac{3}{4}})$ & $1/4$ & adversarial convex bandit \\
%\hline
Thompson Sampling & $\order(\sqrt{TKH(\theta^*)})$ & $1/2$ & stochastic multi-armed bandit \\
\hline
\end{tabular}
\caption{\label{tab:examples} Examples of base algorithms (see Section~\ref{sec:app} and Appendix~\ref{appd:app} for details)}
\end{table*}

\subsection{Exploiting Easy Environments with Robust Guarantee}\label{subsec:app1}
We present two examples below to show how to use {\alg} to create an
algorithm that enjoys some robustness guarantee while being able to
exploit easy environments and perform much better than in the worst
case.

\subsubsection{Contextual Bandits}\label{subsubsec:CB}
%% Recall Example~\ref{ex:cb}, $\Theta$ is a set of policies which map a
%% context $x \in \calX$ to one of the $K$ fixed actions, and the loss
%% function $f_t$ has the form $f_t(\theta, x)
%% = \inner{\bc_t}{\basis_{\theta(x)}}$ where $\bc_t \in [0, 1]^K$ is the
%% loss vector specifying the loss of choosing each action.

We consider the setting of Example~\ref{ex:cb}. It is in general
difficult to derive efficient contextual bandit algorithms without any
assumptions on the set $\Theta$ and the environment. Prior works
usually assume access to an offline ERM oracle that, given a set of
training examples $(x_s, \bc_s)_{s=1,\ldots,t}$, outputs the policy
that minimizes the loss on this training set.  We consider three such
algorithms: ILOVETOCONBANDITS~\citep{AgarwalHsKaLaLiSc14},
BISTRO+~\citep{SyrgkanisLuKrSc16} and the simplest, explore-first
version of Epoch-Greedy~\citep{LangfordZh08}, denoted by
$\base{1}, \base{2}$ and $\base{3}$ respectively.  In addition, we
also consider a classic but inefficient algorithm
EXP4~\citep{AuerCeFrSc02}, denoted by $\base{4}$.  All these base
algorithms satisfy the stability condition but under different
environments, as stated in Lemmas~\ref{lem:mini-monster},
\ref{lem:BISTRO+}, \ref{lem:epoch-greedy}, and~\ref{lem:EXP4}.

Now assuming the context distribution is known, we first combine
$\base{1}$, which exploits the case when losses are also stochastic
(that is, drawn from a conditional distribution $\calD(\cdot|x)$),
and $\base{2}$, which provides a safe guarantee even when the losses
are generated adversarially. The following result is a direct
application of Theorems~\ref{thm:main1} and~\ref{thm:main2}.

\begin{corollary}
Suppose we run {\alg} with two base algorithms:
\textup{ILOVETOCONBANDITS} and \textup{BISTRO+} with learning rate
$\eta = 1/\sqrt{KT\ln|\Theta|}$ and $\Theta = \Theta_1 = \Theta_2$.  
Assuming $x_1, \ldots, x_T$ are
generated independently from a fixed and known distribution, we have:
\begin{packed_enum}
\item \mbox{$
\sup_{\theta \in \Theta} \E\left[ \sum_{t=1}^T f_t(\theta_t, x_t) -
  f_t(\theta, x_t) \right] = \otil\left(%\sqrt{KT\ln|\Theta|} +
(KT)^{\frac{3}{4}}\sqrt{\ln|\Theta|}\right)%^{\frac{1}{4}} \right).
$} for adversarial costs;

\item %% In addition, if the losses $\bc_1, \ldots, \bc_T$ are also generated
%% independently from a fixed family of conditional distributions
%% $\calD(\cdot | x)$ for all $x \in \calX$, then we have
$
\sup_{\theta \in \Theta} \E\left[ \sum_{t=1}^T f_t(\theta_t, x_t) -
  f_t(\theta, x_t) \right] = \otil(\sqrt{KT\ln|\Theta|}),
$~ if $\bc_t \sim \calD(\cdot | x_t)$~ for all $t$.
\end{packed_enum}
\end{corollary}

Next, we combine $\base{1}$, $\base{3}$ and $\base{4}$.  Although seemingly
$\base{3}$ is dominated by $\base{1}$ since it has a worst regret
bound under the same stochastic assumptions, in practice, $\base{3}$
is computationally much faster than $\base{1}$ (indeed, in total $\base{3}$ 
only makes one call of the oracle while $\base{1}$ makes
$\otil(\sqrt{T})$ calls over $T$ rounds), and therefore it can afford to use a more
complicated policy class under the same time constraint.  
Similarly, although $\base{4}$ dominates all other algorithms in terms of regret guarantee,
its running time is linear in the number of policies and can only afford to use a very small policy class.
For example, we can run $\base{4}$ with a policy class of depth-5 decision trees,
$\base{1}$ with a larger policy class of depth-10 decision trees,
and $\base{3}$ with an even larger policy class of depth-20 decision trees.
If the environment is easy in the sense that a depth-5 decision tree
can predict well already, then $\base{4}$ exploits this fact and
achieves $\otil(\sqrt{T})$ regret without any stochastic assumption; 
otherwise, we still have $\base{1}$ to provide $\otil(\sqrt{T})$ regret against a larger class,
and $\base{3}$ to provide $\otil(T^{3/4})$ regret against an even larger class, albeit under i.i.d. assumptions.  
Formally we have the following result:

\begin{corollary}\label{cor:CB}
Suppose we run {\alg} with three base algorithms:
\textup{ILOVETOCONBANDITS} with policy class $\Theta_1$, 
\textup{Epoch-Greedy} with policy class $\Theta_3$,
and \textup{EXP4} with policy class $\Theta_4$
such that $\Theta_4 \subset \Theta_1 \subset \Theta_3$.  If the
learning rate $\eta$ is set to $1/\sqrt{KT\ln|\Theta_4|}$,
then we have:
\begin{packed_enum}
\item $%\[
\sup_{\theta \in \Theta_4} \E\left[ \sum_{t=1}^T f_t(\theta_t, x_t) -
  f_t(\theta, x_t) \right] = \otil(\sqrt{KT\ln|\Theta_4|})
$%\]
~~for adversarial $x_t, \bc_t$;
\item the better of these two bounds if $(x_t, \bc_t)$ are drawn
i.i.d. from a fixed and unknown distribution: 
$%\[
\sup_{\theta \in \Theta_1} \E\left[ \sum_{t=1}^T f_t(\theta_t, x_t) -
  f_t(\theta, x_t) \right] = \otil(\sqrt{KT}\ln|\Theta_1|(\ln|\Theta_4|)^{-\frac{1}{2}}),
$%\]
~~and\\% at the same time:
%\begin{equation*}
$
\sup_{\theta \in \Theta_3} \E\left[ \sum_{t=1}^T f_t(\theta_t, x_t) -
  f_t(\theta, x_t) \right] = \otil\left(%\sqrt{KT\ln|\Theta_1|} +
T^{\frac{3}{4}}K^{\frac{1}{2}} (\ln|\Theta_3|)^{\frac{3}{4}} (\ln|\Theta_4|)^{-\frac{1}{4}} \right).
$%\end{equation*}
\end{packed_enum}
\end{corollary}
%
%In fact, the $T^{3/4}$ factor in Eq.~\eqref{eq:epoch-greedy} can be improved to
%$T^{3/4}\sqrt{c^*}$ where $c^* =\min_{\theta\in\Theta} \E_{(x, \bc)}\left[ \inner{\bc}{\basis_{\theta(x)}} \right]$
%is the optimal expected loss, which can be much smaller than $1$.
%See Appendix~\ref{app:small-loss} for details.

\subsubsection{Convex Bandits}
In convex bandit problems, $\Theta \subset \R^d$ is a compact convex
set (assumed to have constant diameter after rescaling), side
information $x_t$ is usually empty, and the loss function $f_t$ is
assume to be convex in $\theta$.  Without further assumptions, this is
a rather difficult problem.  Recent works~\citep{BubeckDeKoPe15,
  BubeckEl16, BubeckElLe16} make some important progress in this
direction but unfortunately with very complicated and impractical
algorithms.  Here we consider two simpler and more practical
algorithms.  The first one is SCRiBLe~\citep{AbernethyHaRa12} (denoted
by $\base{1}$) which was proposed under the assumption that the $f_t$'s
are linear functions.  The second one is BGD
from~\citep{FlaxmanKaMc05} (denoted by $\base{2}$), which has a regret
guarantee as long as the loss functions are Lipschitz.  
We show that both algorithms admit stability in Lemmas~\ref{lem:SCRiBLe} and~\ref{lem:BGD}.

Again, direct application of Theorems~\ref{thm:main2} now leads to the following more adaptive
algorithm:

\begin{corollary}
Suppose we run {\alg} with two base algorithms: \textup{SCRiBLe} and
\textup{BGD} with learning rate $\eta = \frac{1}{d^{3/2}\sqrt{T}}$.
Assuming $f_1, \ldots, f_T$ are convex and $L$-Lipschitz, we have:
\begin{packed_enum}
\item $%\[
\sup_{\theta \in \Theta} \E\left[ \sum_{t=1}^T f_t(\theta_t, x_t) -
  f_t(\theta, x_t) \right] = \otil\left(L^{\frac{2}{3}}
(dT)^{\frac{5}{6}} \right)
$;%\]
\item in addition, if the losses are linear, that is, $f_t(\theta, x) =
\inner{\theta}{\bc_t}$ for some $\bc_t \in \R^d$, then we have
$%\[
\sup_{\theta \in \Theta} \E\left[ \sum_{t=1}^T f_t(\theta_t, x_t) -
  f_t(\theta, x_t) \right] = \otil(d^{\frac{3}{2}}\sqrt{T}).
$%\]
\end{packed_enum}
\end{corollary}

\subsection{Robustness to Many Different Environments}\label{subsec:app2}
\label{sec:models}
Another application of {\alg} is to create an algorithm that works
simultaneously under different environments and can select the correct
model automatically.  Although the dependence on the number of base
algorithms is polynomial instead of
logarithmic as is usually the case for model selection problems,
there are still many scenarios where the number of models is
relatively small and polynomial dependence is not a serious problem.
We present several examples below.

\subsubsection{Multi-Armed Bandits} 
The classic $K$-armed bandit problem is simply the case where $\Theta
= [K]$ and $f_t(\theta, x) = \inner{\bc_t}{\basis_{\theta}}$.
Although there exist algorithms (such as EXP3~\citep{AuerCeFrSc02})
that guarantee the optimal $\otil(\sqrt{TK})$ regret even if the
losses are generated adversarially, in practice, a Bayesian approach
called Thompson Sampling~\citep{Thompson33} is often used and known to
perform well.  However, like other Bayesian approaches, Thompson
Sampling assumes a prior over the environments, and the regret
guarantee is usually only meaningful when the prior is true.\footnote{%
  There is also worst-case analysis for Thompson Sampling; see for
  example~\citep{AgrawalGo12, KaufmannKoMu12}.  }

Nevertheless, with our ensemble approach, one can easily create an
algorithm that works under different true priors.  To present the
results, we follow the analysis of~\citep{RussoVa14}.  Suppose the
loss vectors $\bc_t$ are i.i.d samples of a distribution $\calD$ which
is itself drawn from a prior distribution $\calP$ over a family of
distributions. Let $\mu_\calD = \E_{\bc \sim\calD}[\bc]$ be the mean
vector drawn from the prior and $q_i =
\Pr_{\calD\sim \calP}\left(\mu_{\calD,i} \leq \mu_{\calD, j},~~\forall~j \in [K]\right)$ so that $\bq$ is the distribution of the optimal arm. Let $H(\bq)$ be the entropy of $\bq$. Then we
have the following results (note that all expectations are taken with
respect to the true prior in addition to all other randomness):

\begin{corollary}
If we run {\alg} with $M$ instances of Thompson Sampling, each of
which uses a different prior $\calP_i$, and the true prior $\calP
= \calP_{i^\star}$ for some $i^\star$, then with $\eta = \sqrt{\frac{M}{TK}}$ we
have\\
$%\[
\sup_{\theta \in \Theta} \E\left[ \sum_{t=1}^T f_t(\theta_t, x_t) -
  f_t(\theta, x_t) \right] = \otil(\sqrt{MTK}H(\bq^\star)),
$%\]
where $\bq^\star$ is the distribution over the optimal arm induced by $\calP_{i^\star}$.
\end{corollary}

\subsubsection{Other Examples}
We briefly mention some other examples without giving details.  For
contextual bandits, if we
have different ways to represent the contexts, then each base
algorithm can be any existing contextual bandit algorithm with a
specific context representation and policy space.  The master can then
have good performance as long as one of these representations captures
the problem well.

For stochastic linear bandits, $\Theta \subset \R^d$ is a compact
convex set and $f_t(\theta, x) = \inner{\theta}{\bc^*} + \xi_t$ where
$\bc^* \in \R^d$ is fixed and unknown, and $\xi_t$ is some zero-mean
noise. Previous works have studied cases where $\bc^*$ is assumed to
admit some special structures, such as sparsity, group-sparsity and so
on (see for example~\citep{AbbasiYadkoriPaSz12, CarpentierMunos12,
JohnsonSiBa16}). One can then run {\alg} with different base
algorithms assuming different structures of $\bc^*$.  Another related
problem is generalized linear bandits, where $f_t(\theta, x) =
\sigma(\inner{\theta}{\bc^*}) + \xi_t$ for some link function $\sigma$
(such as the logistic function, exponential function and so on,
see~\citep{FilippiCaGaSz10}).  It is clear that one can run {\alg}
with different base algorithms using different link functions to
capture more possibilities of the environments.  In all these cases,
the number of base algorithms is relatively small.

\section{Conclusion and Open Problems}\label{sec:conclusion}
In this work, we presented a master algorithm which can combine a set of
base algorithms and perform as well as the best of them in a very
strong sense in the bandit setting.  Two major applications of our
approach were presented to illustrate how this master algorithm can be
used to create more adaptive bandit algorithms in a black-box fashion.

There are two major open problems left in this direction.  One is to
improve the results of Theorem~\ref{thm:main2} so that the master can
basically inherit the same regret bounds of all the base
algorithms, i.e., Eq.~\eqref{eqn:desired} holds
simultaneously for all base algorithms satisfying stability
condition with $\alpha_i < 1$.  Note that this is in
general impossible (see~\citep{LattimoreSz16} for a lower bound in a
special case), but it is not clear whether it is possible if we only
care about the scaling with $T$ while allowing worse dependence on
other parameters.  The current approach fails to achieve this mainly
because each of these bounds requires a different tuning of the same
learning rate $\eta$.

Another open problem is to
improve the dependence on $M$, the number of base algorithms, from
polynomial to logarithmic while keeping the same dependence
on other parameters (or prove its impossibility).  
Logarithmic dependence on $M$ can be achieved by using EXP4 as the master,
but as was earlier discussed, this leads to poor dependence on other parameters. 

\ifcolt
\acks{
The authors would like to thank John Langford for posing the question initially that stimulated this research.
Most of the work was completed when Behnam Neyshabur was an intern at Microsoft Research.
}
\else
\paragraph{Acknowledgments.}
The authors would like to thank John Langford for posing the question initially that stimulated this research.
Most of the work was completed when Behnam Neyshabur was an intern at Microsoft Research.
\fi

%\newpage
\bibliographystyle{plainnat}
\bibliography{ref}

%\newpage
\appendix

%\begin{claim}\label{clm:loss_range}
%For any $i \in [M]$ and $t \in [T]$, we have $\frac{1}{\pbar_{t,i}} \leq \scale_{t,i}$.
%\end{claim}
%
%\begin{proof}[Proof of Claim~\ref{clm:loss_range}]
%The statement clearly holds when $t=1$. 
%Assuming it holds for $t \geq 1$, by Line~\ref{line:doubling} of Algorithm~\ref{alg:main}, 
%we have either $\frac{1}{\pbar_{t+1,i}} \leq \scale_{t,i} = \scale_{t,i}$,
%or $\frac{1}{\pbar_{t+1,i}} > \scale_{t,i}$.
%In the second case, we have by Algorithm~\ref{alg:LOMD},
%$\frac{1}{p_{t+1,i}} \leq \frac{1}{p_{t,i}} + \eta_{t,i} \loss_{t,i} \leq \frac{1+\eta_{t,i}}{p_{t,i}}$
%and therefore 
%\[ \frac{1}{\pbar_{t+1,i}} = \frac{1}{(1-\gamma)p_{t+1,i} + \frac{1}{TM}} 
%\leq \frac{1+\eta_{t,i}}{(1-\gamma)p_{t,i} + \frac{1}{TM}} = \frac{1+\eta_{t,i}}{\pbar_{t,i}} 
%\leq (1+ e^{\frac{\log_2 T}{\ln T}}\eta) \scale_{t,i} \leq 2\scale_{t,i} = \scale_{t+1,i} 
%\]
%where the last inequality holds whenever $\eta \leq 0.2$. This completes the induction.
%\end{proof}

\section{More Examples of the Setting}
\label{app:examples}
In this section, we instantiate the notions of a decision, environment and bandit
algorithm in our general setting for several concrete examples.
These examples are meant to illustrate the generality intended by our
setup, suggesting the potentially broad consequences of results
obtained in it.
We start with the most basic setting.

\begin{example}[Multi-armed bandits]
\label{ex:mab}
\textup{
Multi-armed bandits~\citep{lai1985asymptotically} is the simplest
instantiation of our setup where there is no side-information $x_t$
and the decision space is a set of $K$ arms, that is, $\Theta = [K]$,
and the loss function specifies the loss of pulling each arm so that
$f_t(\theta, x) = \inner{\bc_t}{\basis_{\theta}}$ for some $\bc_t 
\in [0,1]^K$. There are two main types of environments for which 
algorithms have been developed for this problem: 
}

\textbf{Stochastic environment:} 
\textup{In this case, the loss vectors $\bc_t$ are
independent random draws from some fixed distribution at each round, 
and the environment is fully characterized by this fixed distribution.
Perhaps the most well-known algorithm in this setting is the UCB
strategy~\citep{auer2002finite} which obtains an expected regret of at
most $\otil(\sqrt{KT})$.\footnote{We skip the discussion of more
  detailed gap-dependent bounds here.}
}

\textbf{Adversarial environment:} 
\textup{A significantly harder setting is
one where the loss vectors $\bc_t$ are chosen arbitrarily by an adaptive
adversary. That is, the environment is an arbitrary mapping from the history
$(\theta_s, \bc_s)_{s=1, \ldots, t-1}$ to the next loss vector $\bc_t$.
The EXP3 algorithm of~\citet{AuerCeFrSc02} is an approach
which gets an expected regret of $\otil(\sqrt{KT})$ in this harder
setting.
}

\textup{
There are several modifications and refinements of this basic setting
which we skip over here. For instance, the stochastic environments
have been further refined to when the expected losses of the best arm
are substantially lower than the rest. In the adversarial setting,
there are results that take advantage of the loss functions changing
slowly, or having a budget on the total amount of change an adversary
can induce, in order to get better results. Our subsequent results
will potentially allow us to enjoy better guarantees in some of these
special cases, while remaining robust in the worst case.
}
\end{example}

\begin{example}[Contextual bandits, expanded version of Example~\ref{ex:cb}]
\label{ex:cb-app}
\textup{
Contextual bandits~\citep{LangfordZh08} is a generalization of the
multi-armed bandit problem where the side information $x_t$ (called a
context) is non-empty. The learner's decision space $\Theta$ consists
of a set of policies, where a policy maps contexts to a discrete set
of actions, i.e. $\theta~:~X\mapsto [K]$. For instance, if the contexts are points in $\R^d$,
a policy might be parametrized by a weight matrix $\bW \in \R^{K\times d}$ so that
$\theta(x) = \arg\max_{i \in [K]} \bW_i x$ where $\bW_i$ is the $i$-th
row of $\bW$. 
Different base algorithms can in general work with very different policy classes,
which can be captured by different $\Theta_i$ in our setting.
The loss function is
again in the form $f_t(\theta, x) = \inner{\bc_t}{\basis_{\theta(x)}}$ for some $\bc_t \in [0,1]^K$.
This problem has been studied under three
main environments:
}

\textbf{Stochastic contexts and losses:} 
\textup{ In the simplest instance,
both the contexts $x_t$ and losses $\bc_t$ are drawn i.i.d. according
to a fixed distribution, and the environment is characterized by this
distribution. The Epoch-Greedy algorithm of~\citet{LangfordZh08}
suffers an expected regret of $\otil(T^{2/3})$ in this setting. The
more recent work of~\citet{AgarwalHsKaLaLiSc14} has a better regret
bound of $\otil(\sqrt{T})$, though at a significantly higher
computational cost. 
}

\textbf{Adversarial contexts or losses:} 
\textup{Several authors~\citep{auer2002using, ChLiReSc11, FilippiCaGaSz10} have
studied environments where the contexts are chosen by an adversary,
but the losses come from a fixed, parametric form such as $\E[\bc_t] =
\bW^\star x_t$ where $\bW^\star$ is some fixed, unknown weight
matrix. Thus the environment is characterized by the adversarial
strategy for picking the next context given the history, along with
the conditional distribution of losses given the context. While this
relaxes the i.i.d. assumption on the contexts in the first setting, it
places a more restrictive model on the stochastic losses. Algorithms
such as LinUCB and variants~\citep{LiChLaSc10, ChLiReSc11} enjoy
$\otil(\sqrt{T})$ regret in these settings. Other
authors~\citep{SyrgkanisLuKrSc16, rakhlin2016bistro} have studied
settings where the contexts are i.i.d. from a fixed distribution, but
the losses are picked in an adversarial manner, a strict
generalization of the i.i.d. setting from
above. \citet{SyrgkanisLuKrSc16} have proposed an algorithm which
suffers an expected regret of at most $\otil(T^{2/3})$ in this
setting.
}

\textbf{Adversarial contexts and losses:} 
\textup{This is the hardest
environment which was addressed in an early work
of~\citet{AuerCeFrSc02}, who propose the EXP4 algorithm. This
algorithm incurs an expected regret at most $\otil(\sqrt{T})$ in the
most general setting, but is computationally inefficient.
}
\end{example}

\begin{example}[Convex bandits]
\textup{
This setting is a different way of generalizing the multi-armed bandit
problem, and was initiated by the work of~\citet{FlaxmanKaMc05}. In
this setting, the side-information $x_t$ is again empty. The decision
space $\Theta$ is typically some convex, compact subset of $\R^d$ such
as a ball in a chosen norm. The loss functions $f_t$ are typically
convex (in $\theta$) functions with some added regularity conditions. The two most
well-studied settings here are:
}

\textbf{Adversarial linear functions:} 
\textup{In this case, the loss
functions are linear, that is $f_t(\theta, x) = \inner{\bc_t}{\theta}$
where each $\bc_t \in \R^d$ is chosen by an adversary.
\citet{AbernethyHaRa12} present an algorithm for this setting with an
expected regret of $\otil(d^{3/2}\sqrt{T})$. Several authors have also
improved the regret bound for specific sets $\Theta$ as well as when
the loss vectors are i.i.d (e.g.~\citep{BubeckCeKa12}).
}

\textbf{Adversarial convex functions:} 
\textup{More generally, the loss
functions can be general convex, Lipschitz-continuous functions of
$\theta$, with the environment described by the adversary's strategy
for picking the next loss function given the
history. \citet{FlaxmanKaMc05} develop an algorithm which incurs an
expected regret at most $\order(dT^{3/4})$. These results have been
refined in subsequent works making further smoothness and strong
convexity assumptions on the loss functions, as well as to
$\otil(d^{9.5}\sqrt{T})$ regret in the more general setting in a very
recent work of~\citet{BubeckElLe16}.
}
\end{example}

Thus we see that several prior works on learning with partial feedback
are admissible under our model. We again highlight that this is only a
very quick survey of a large body of literature, and we are omitting
discussion of many other setups such as Lipschitz losses in a metric
space, gap-dependent results in stochastic settings etc., all of which
are also fully captured in our setting.

\section{Proof of Theorem~\ref{thm:lower_bound}}
\label{app:lower_bound}
\begin{proof}(\emph{Sketch}) Consider a simple setting where there are 2
  base algorithms $\base{1}$ and $\base{2}$, a total of 4 actions
  with $\Theta_1 = \Theta_2 = \{a_1, a_2, a_3, a_4\}$,
and 2 possible environments $\calE_1$ and $\calE_2$, both of which assign (unknown) fixed losses 
to the actions so that each action deterministically yields its assigned loss every round.
More specifically, in $\calE_1$, the losses assigned to $a_1$ and $a_2$ are $0.1$ and $0.2$
or $0.2$ and $0.1$ with equal probability,
and the losses assigned to $a_3$ and $a_4$ are $0.3$ and $0.4$
or $0.4$ and $0.3$ with equal probability.
Similarly for $\calE_2$, the situation are reversed so that $a_1$ and $a_2$ are always worse than $a_3$ and $a_4$.

Base algorithms $\base{1}$ and $\base{2}$ are two nearly-identical
copies of the same simple algorithm designed specially for these two environments. 
Specifically, $\base{1}$ pulls $a_1$ in the first round.
If the observed loss is $0.1$ or $0.3$, it keeps playing $a_1$;
if the observed loss is $0.2$ or $0.4$, it keeps playing $a_2$;
any other observed loss will lead to uniformly random choices between $a_1$ and $a_2$ for the rest of the game.
$\base{2}$ is similar except it only plays $a_3$ and $a_4$ in the same fashion.
Clearly, $\base{1}$ has constant regret in $\calE_1$ while $\base{2}$ has constant regret in $\calE_2$.

Now the claim is that for any master, its expected regret must be $\Omega(T)$ under either $\calE_1$ or $\calE_2$.
This is because without the knowledge of which environment it is in,
in the first round the master will inevitably follows the ``wrong'' base algorithm
(that is, $\base{1}$ in $\calE_2$ or $\base{2}$ in $\calE_1$) with constant probability under either $\calE_1$ or $\calE_2$. 
Without loss of generality, assume $\calE_2$ is the true environment and the master 
follows $\base{1}$ and thus plays $a_1$ in the first round. It can then supply the right feedback
to $\base{1}$; however, it has no information about the loss of
$a_3$, the action that $\base{2}$ suggested, and therefore it fails to update $\base{2}$ correctly 
and as a result $\base{2}$ will choose the wrong action with constant probability for the rest of the rounds.
This means that the master has no way of recovering from this error and picks up linear regret.
\end{proof}

\section{Proofs of Main Results}
\label{app:analysis}

We start by stating a regret guarantee for {\LOMD}, whose proof mostly
follows the standard analysis (see for
example~\citep{Shalevshwartz11}) except for the part involving the
special log barrier mirror map (which is also the part that is
slightly different from~\citep{FosterLiLySrTa16}). 
We recall our earlier notation
$$h(y) = y - 1 - \ln(y), ~~\mbox{so that}~~
\bregman{\psi_t}{\bp}{\bq} = \sum_{i=1}^M \frac{1}{\eta_{t,i}}
h\left(\frac{p_i}{q_i}\right).
$$

\begin{lemma}\label{lem:LOMD}
{\LOMD} ensures that for any $\bu \in \Delta_M$, we have after $T$ rounds
\[
\sum_{t=1}^{T} \inner{\bp_t - \bu}{\bloss_t} \leq 
\sum_{t=1}^{T} \paren{\bregman{\psi_t}{\bu}{\bp_t} - \bregman{\psi_t}{\bu}{\bp_{t+1}}}  +  \sum_{t=1}^{T} \sum_{i=1}^M \eta_{t,i} p_{t,i}^2 \loss_{t,i}^2 \;.
\]
\end{lemma}

\begin{proof}%[Proof of Lemma~\ref{lem:LOMD}]
For any $\bu \in \Delta_M$, by the algorithm, direct calculations and the generalized Pythagorean theorem, we have for any $t$:
\begin{align*}
\inner{\bp_t - \bu}{\bloss_t} &= \inner{\bp_t - \bu}{\nabla \psi_t(\bp_t) - \nabla\psi_t(\bptilde_{t+1}) }  \\
&= \bregman{\psi_t}{\bu}{\bp_t} - \bregman{\psi_t}{\bu}{\bptilde_{t+1}} + \bregman{\psi_t}{\bp_t}{\bptilde_{t+1}} \\
&\leq \bregman{\psi_t}{\bu}{\bp_t} - \bregman{\psi_t}{\bu}{\bp_{t+1}} + \bregman{\psi_t}{\bp_t}{\bptilde_{t+1}}.
\end{align*}
It thus remains to prove $h\left(\frac{p_{t,i}}{\ptilde_{t+1,i}}\right) \leq \eta_{t,i}^2 p_{t,i}^2\loss_{t,i}^2$.
Notice that by the algorithm, we have $\frac{p_{t,i}}{\ptilde_{t+1,i}} = 1 + \eta_{t,i} p_{t,i}\loss_{t,i}$.
Therefore by the definition of $h(y)$ and the fact $\ln(1 + x) \geq x - x^2$ when $x \geq 0$ we arrive at
\[
h\left(\frac{p_{t,i}}{\ptilde_{t+1,i}}\right)  = 
\frac{p_{t,i}}{\ptilde_{t+1,i}} -1 - \ln\left(\frac{p_{t,i}}{\ptilde_{t+1,i}}\right)
= \eta_{t,i} p_{t,i}\loss_{t,i} - \ln\left(1 + \eta_{t,i} p_{t,i}\loss_{t,i}\right)  \leq \eta_{t,i}^2 p_{t,i}^2\loss_{t,i}^2,
\]
which completes the proof.
\end{proof}

Next, we use the above lemma along with the sophisticated learning rates schedule to give
a bound on the master's regret to any base algorithm.
Importantly, the bound includes a negative term that is in terms of $\scale_{T,i}$.

\begin{lemma}\label{lem:master2base}
{\alg} ensures that for any $i \in [M]$, we have
\[
\E\left[ \sum_{t=1}^T f_t(\theta_t, x_t) -  f_t(\theta_t^i, x_t)  \right]  \leq 
\order\left(\frac{M\ln T}{\eta} + T\eta \right) - \E\left[\frac{\scale_{T,i}}{40\eta\ln T}\right]  
\]
\end{lemma}
\begin{proof}
Fix all the randomness, let $n_i$ be such that $\eta_{T,i} = \beta^{n_i} \eta$,
where we assume $n_i \geq 1$ (the case $n_i = 0$ is trivial as one will see).
Let $t_1, \ldots, t_{n_i}$ be the rounds where Line~\ref{line:doubling} is executed for base algorithm $\base{i}$.
Since $\frac{1}{\pbar_{t_{n_i}+1, i}} > \scale_{t_{n_i}, i} > 2\scale_{t_{n_i-1}, i} > \ldots > 2^{n_i} M$ and 
$\frac{1}{\pbar_{t, i}} \leq TM$ for any $t$ by Line~\ref{line:mixing}, we have $n_i \leq \log_2 T$.

It is clear that {\alg} is running {\LOMD} with $\bloss_t = \frac{f_t(\theta_t, x_t)}{\pbar_{t,i_t}} \basis_{i_t}$.
We can therefore apply Lemma~\ref{lem:LOMD}, focusing on the term
$\sum_{t=1}^{T}\bregman{\psi_t}{\bu}{\bp_t} - \bregman{\psi_t}{\bu}{\bp_{t+1}}$.
By the fact that Bregman divergence is non-negative and the learning rate $\eta_{t,j}$ for each $j \in [M]$ is non-decreasing in $t$,
we have
\begin{align*}
\sum_{t=1}^{T}\bregman{\psi_t}{\bu}{\bp_t} - \bregman{\psi_t}{\bu}{\bp_{t+1}}
&\leq \bregman{\psi_1}{\bu}{\bp_1} +
\sum_{t=1}^{T-1} \paren{
  \bregman{\psi_{t+1}}{\bu}{\bp_{t+1}} - \bregman{\psi_t}{\bu}{\bp_{t+1}} }  \\
&= \bregman{\psi_1}{\bu}{\bp_1} + \sum_{t=1}^{T-1} \sum_{j=1}^M \left(
\frac{1}{\eta_{t+1,j}} -  \frac{1}{\eta_{t,j}}  \right)
h\left(\frac{u_j}{p_{t+1,j}}\right) \tag{recall $h(y) \geq 0$}\\
%&\leq \bregman{\psi_1}{\bu}{\bp_1} + \sum_{t=1}^{T-1} \left( \frac{1}{\eta_{t+1,i}} -  \frac{1}{\eta_{t,i}}  \right)  h\left(\frac{u_i}{p_{t+1,i}}\right) \\
&\leq \bregman{\psi_1}{\bu}{\bp_1} + \left( \frac{1}{\eta_{t_{n_i}+1,i}} -  \frac{1}{\eta_{t_{n_i},i}}  \right)  h\left(\frac{u_i}{p_{t_{n_i}+1,i}}\right) \\
&= \bregman{\psi_1}{\bu}{\bp_1} + \frac{1-\beta}{\beta^{n_i}\eta } h\left(\frac{u_i}{p_{t_{n_i}+1,i}}\right)  \\
&\leq \bregman{\psi_1}{\bu}{\bp_1} - \frac{1}{5\eta\ln T}  h\left(\frac{u_i}{p_{t_{n_i}+1,i}}\right) 
\end{align*}
where the last step is by the fact $1 - \beta \leq -\frac{1}{\ln T}$ and $\beta^{n_i} \leq e^{\frac{\log_2 T}{\ln T}} \leq 5$.

We now set $\bu = (1-\frac{1}{T}) \basis_i + \frac{1}{TM} \one \;\in \Delta_M$.
Assuming $T \geq 2$ we have $u_i \geq 1 - \frac{1}{T} \geq \frac{1}{2}$ and $u_j \geq \frac{1}{TM}$ for all $j$.
We can thus bound the first term as
\[
\bregman{\psi_1}{\bu}{\bp_1} = \sum_{j=1}^M \frac{1}{\eta_{1,j}} h\left(\frac{u_j}{p_{1,j}}\right)  
= \frac{1}{\eta} \sum_{j=1}^M  h(M u_j) = \frac{1}{\eta} \sum_{j=1}^M  \ln\left(\frac{1}{M u_j}\right) \leq \frac{M \ln T}{\eta}.
\]
For the second term, note that we have $\frac{u_i}{p_{t_{n_i}+1,i}} \geq \frac{1}{4\pbar_{t_{n_i}+1,i}}
\geq 2^{n_i-2} M \geq 1$ as long as $M \geq 2$.
So with the facts that $h(y)$ is increasing when $y \geq 1$ and $\scale_{T,i} = \frac{2}{\pbar_{t_{n_i}+1,i}}$, we have
\[
h\left(\frac{u_i}{p_{t_{n_i}+1,i}}  \right)  \geq h\left(\frac{1}{4\pbar_{t_{n_i}+1,i}}\right)  
= \frac{\scale_{T,i}}{8} - 1 - \ln\left(\frac{1}{4\pbar_{t_{n_i}+1,i}}  \right)
\geq \frac{\scale_{T,i}}{8} - 1 - \ln\left(\frac{TM}{4}  \right)
\]
and therefore
\[
\sum_{t=1}^{T}\bregman{\psi_t}{\bu}{\bp_t} - \bregman{\psi_t}{\bu}{\bp_{t+1}} \leq O\left(\frac{M \ln T}{\eta}\right) - \frac{\scale_{T,i}}{40\eta \ln T}.
\]

Finally, with the definition of $\bloss_t$ and the facts 
\[
\inner{\bp_t - \bu}{\bloss_t} \geq \inner{(1-\gamma)\bp_t - \bu}{\bloss_t} 
= \inner{\bpbar_t - \basis_i}{\bloss_t} + \inner{\basis_i - \bu - \frac{\gamma}{M}\one}{\bloss_t}
\geq \inner{\bpbar_t - \basis_i}{\bloss_t} - \frac{2 \loss_{t, i_t}}{TM}
\]
and $\sum_{j=1}^M\eta_{t, j} p_{t,j}^2 \loss_{t,j}^2 \leq \eta_{t, i_t} \frac{p_{t,i_t}^2}{\pbar_{t,i_t}^2} \leq \eta \frac{\beta^{\log_2 T}}{(1-\gamma)^2} \leq 20\eta$,
together with Lemma~\ref{lem:LOMD}, we arrive at
\[
\sum_{t=1}^T \inner{\bpbar_t - \basis_i}{\bloss_t} 
\leq O\left(\frac{M \ln T}{\eta} + T\eta \right) + \left(\sum_{t=1}^T  \frac{2 \loss_{t, i_t}}{TM}\right) - \frac{\scale_{T,i}}{40\eta \ln T}.
\]
Note that the conditional expectation of $\ell_{t,j}$ with respect to
the random draw of $i_t$ is $f_t(\theta_t^j, x_t)$ for all $j\in [M]$
and the conditional expectation of $\ell_{t,i_t}$ is $\sum_{j=1}^M
f_t(\theta_t^j, x_t) \leq M$. Also $$\E[\inner{\bpbar_t}{\bloss_t}] =
\sum_{i=1}^M \pbar_{t,i} f_t(\theta_t^i,x_t) =
\E[f_t(\theta_t,x_t)].$$ Taking the expectations then finishes the
proof.
\end{proof}

With the tool of Lemma~\ref{lem:master2base}, the proofs of Theorem~\ref{thm:main-meta}, \ref{thm:main1} and~\ref{thm:main2}
are simply to decompose the regret of the master and to make use of the negative term to cancel the large regret 
of the base algorithm in some sense.

\begin{proof}[\ifcolt\else Proof\fi of Theorem~\ref{thm:main-meta}, \ref{thm:main1} and~\ref{thm:main2}]
We begin by splitting the regret into two parts,
namely the regret of the master to $\base{i}$ and the regret of $\base{i}$ to a fixed point in $\Theta_i$:
\begin{align*}
&\sup_{\theta \in \Theta_i}  \E\left[ \sum_{t=1}^T f_t(\theta_t, x_t) -  f_t(\theta, x_t) \right] \\
=&\; \E\left[ \sum_{t=1}^T f_t(\theta_t, x_t) -  f_t(\theta_t^i, x_t)  \right] 
 + \sup_{\theta \in \Theta_i}  \E\left[ \sum_{t=1}^T f_t(\theta_t^i, x_t) -  f_t(\theta, x_t) \right] \\
=&\; \E\left[ \sum_{t=1}^T f_t(\theta_t, x_t) -  f_t(\theta_t^i, x_t)  \right] 
 + \sup_{\theta \in \Theta_i}  \E\left[ \sum_{t=1}^T f_t^i(\theta_t^i, x_t) -  f_t^i(\theta, x_t) \right],
\end{align*}
where we use the fact that the fake loss function $f_t^i$ is an unbiased estimator of the true loss function $f_t$ by construction.
Theorem~\ref{thm:main-meta} then follows directly by the stability condition of $\base{i}$ and Lemma~\ref{lem:master2base}.

Next, setting $\alpha_i=1$ and $\eta = \min\left\{\frac{1}{40\reg_i(T)\ln T},
\sqrt{\frac{M}{T}}\right\}$ we have
\begin{align*}
\sup_{\theta \in \Theta_i}  \E\left[ \sum_{t=1}^T f_t(\theta_t, x_t) -  f_t(\theta, x_t) \right]
&\leq \otil\left(\frac{M}{\eta} + T\eta \right) - \E\left[\frac{\scale_{T,i}}{40\eta\ln T}\right]  + \E[\scale_{T,i}]\reg_i(T) \\
&\leq \otil\left(\sqrt{MT} + M\reg_i(T)\right) - \E[\scale_{T,i}]\reg_i(T)  + \E[\scale_{T,i}]\reg_i(T) \\
&= \otil\left(\sqrt{MT} + M\reg_i(T) \right),
\end{align*}
proving Theorem~\ref{thm:main1}.

On the other hand, when $\alpha_i < 1$ we have
\begin{align*}
\sup_{\theta \in \Theta_i}  \E\left[ \sum_{t=1}^T f_t(\theta_t, x_t) -  f_t(\theta, x_t) \right]
&\leq \otil\left(\frac{M}{\eta} + T\eta \right) + \E\left[ \scale_{T,i}^{\alpha_i} \reg_i(T) - \frac{\scale_{T,i}}{40\eta\ln T}\right]   \\
&\leq \otil\left(\frac{M}{\eta} + T\eta + \reg_i(T)^{\frac{1}{1-\alpha_i}} \eta^{\frac{\alpha_i}{1-\alpha_i}} \right) 
\end{align*}
where the last step is by maximizing the function $\scale^{\alpha_i}\reg_i(T) - \frac{\scale}{40\eta\ln T}$ over $\scale > 0$.
This proves Theorem~\ref{thm:main2}.
\end{proof}

\section{Omitted Details for Section~\ref{sec:app}}\label{appd:app}

To verify the stability condition for the base algorithms,
by a standard doubling trick argument it suffices to verify the following weaker 
version of the condition where the algorithm
knows a bound on the loss range $\scale$ ahead of time:

\begin{definition}\label{def:weak_stable}
For some constant $\alpha \in(0, 1]$ and non-decreasing function $\calR: \N_+ \rightarrow \R_+$, 
an algorithm with decision space $\Theta_0 \subset \Theta$ is called $(\alpha, \calR)$-weakly-stable 
with respect to an environment $\calE$
if its regret under $\calE$ is $\calR(T)$, and its regret under any induced environment $\calE'$ is
\begin{equation}
  \sup_{\theta \in \Theta_0}~\E\left[\sum_{t=1}^T f_t'(\theta_t, x_t) -
      f_t'(\theta, x_t) \right] \leq  \scale_T^{\alpha} \reg(T)
\label{eqn:weak_stability}
\end{equation}
where $\scale_T \geq \scale = \max_{t \in [T]} 1/p_t$ is given to the algorithm 
ahead of time, and all expectations are taken over the randomness of both $\calE'$ and the algorithm.
\end{definition}

In fact, here we can even incorporate the doubling trick nicely into \alg as described below.
Suppose that the base algorithms take a loss-range
parameter as input upon initialization.  At the beginning, we initialize
these algorithms with range parameter $\scale_{1,i} = 2M$.  Moreover, we
do an extra step in Line~\ref{line:doubling} of \alg: restart base algorithm
$\base{i}$ with range parameter $\scale_{t+1,i}$.  It is clear that the
losses that any instances of the base algorithms receive will not
exceed their range parameter. We call each of these reruns of a base algorithm 
an instantiation of that algorithm.

Now the following theorem shows that the weak stability condition is enough to 
show all our results up to constants.
(Note that instead of directly showing the stability defined in Definition~\ref{def:stable},
we prove Eq.~\eqref{eq:cond} which is all we need from the stability condition).

\begin{theorem}\label{thm:doubling}
If base algorithm $\base{i}$ is $(\alpha_i, \calR_i)$-weakly-stable with respect to an environment
$\calE$, then running \alg (with the above modification) under $\calE$ ensures
\begin{equation}\label{eq:cond}
  \sup_{\theta \in \Theta_i} \E\left[ \sum_{t=1}^T f_t^i(\theta_t^i, x_t)
    - f_t^i(\theta, x_t) \right] \leq \frac{2^{\alpha_i}}{2^{\alpha_i} - 1} \E\left[\scale_{T,i}^{\alpha_i}\right] \reg_i(T)
\end{equation}
\end{theorem}

%If for some environment there exists a base algorithm
%$\base{i}$, some function $\reg_i: \N_+ \mapsto \R_+$ and 
%some constant $0 < \alpha_i \leq 1$,  
%such that for any instantiation of the algorithm, conditioning
%on the past and letting $t_1$ and $t_2$ be the first and (random) last
%round of this instantiation, we have
%\begin{equation}\label{eq:weakcond}
%\sup_{\theta \in \Theta_i} \E\left[ \sum_{t=t_1}^{t_2}
%  f_t^i(\theta_t^i, x_t) - f_t^i(\theta, x_t) \right] \leq
%\scale_{t_1,i}^{\alpha_i} \reg_i(T)
%\end{equation}
%almost surely, then condition~\eqref{eq:cond} in Theorem~\ref{thm:main-meta} holds (up to constants).

\begin{proof}
Reusing notation from Section~\ref{app:analysis}, let $t_1, \ldots, t_{n_i} < T$ be the rounds where
Line~\ref{line:doubling} is executed.
Also let $t_0 = 0$ and $t_{n_i + 1} = T$ for notational convenience. 
Note that the entire game is divided into $n_i + 1 \leq \lceil \log_2 T \rceil$ segments $[t_{k-1} + 1, t_{k}]$ for $k = 1, \ldots, n_i+1$
based on the restarting of $\base{i}$. We then have by the weak stability condition and monotonicity of $\calR_i$,

\begin{align*}
&\E\left[ \sum_{t=1}^T f_t^i(\theta_t^i, x_t) -  f_t^i(\theta, x_t) \right]  \\
=&\; \sum_{k=1}^{\lceil \log_2 T \rceil} \Pr[n_i +1 \geq k] \cdot
\E\left[ \sum_{t = t_{k-1}+1}^{t_k} f_t^i(\theta_t^i, x_t) -  f_t^i(\theta, x_t)  \;\Bigg|\; n_i +1 \geq k \right] \\
\leq&\; \reg_i(T) \sum_{k=1}^{\lceil \log_2 T \rceil} \Pr[n_i + 1\geq k] \cdot \E\left[\scale_{t_k, i}^{\alpha_i} \;\big|\; n_i +1 \geq k \right]  \\
=&\; \reg_i(T) \E\left[ \sum_{k=1}^{n_i+1} \scale_{t_k, i}^{\alpha_i}  \right] .
\end{align*}
Now let $S = \sum_{k=1}^{n_i+1}  \scale_{t_k, i}^{\alpha_i}$ and note that
\[ 
(2^{\alpha_i}-1) S =  \sum_{k=1}^{n_i} \left((2\scale_{t_k, i})^{\alpha_i} - \scale_{t_{k+1}, i}^{\alpha_i}\right) 
+  (2\scale_{t_{n_i+1},i})^{\alpha_i} - \scale_{t_1, i}^{\alpha_i}
\leq (2\scale_{t_{n_i+1},i})^{\alpha_i} = 2^{\alpha_i}\scale_{T,i}^{\alpha_i}.
\]
where we use the fact $2\scale_{t_k, i} \leq \scale_{t_{k+1}, i}$.
This proves the theorem.
\end{proof}

In the following subsections, we prove that weak stability holds
for different algorithms discussed in Section~\ref{sec:app} as listed in Table~\ref{tab:examples}.  
Note that when we say that an algorithm satisfies the condition, we
always mean that with appropriate parameters or even slight
modifications it satisfies the condition. 
%
%Another subtlety
%  is that the behavior of the base algorithms actually depends on the
%  learning rate $\eta$ of {\alg}.  However, the conditions are
%  satisfied no matter what value of $\eta$ is used. 
%
%We will always consider a specific instance of a fixed base algorithm
%$\base{i}$.  So for notational convenience, we can let $t_1 = 1$ and
%$t_2 = T$ (this is not to be confused with the total horizon $T$ which we never need to refer to in the remainder).  
%Note that we can even assume $T$ is fixed, although
%it is in fact a random variable, since base algorithms can simply
%apply a standard doubling trick to deal with the case where $T$ is
%unknown.
%
Moreover, for notation convenience we drop the subscript for $\scale_T$
(which is overloading the notation $\scale$ but they convey similar meanings anyway and there will not be confusion below),
and define random variable $s_t$ which is $1/p_t$ with probability $p_t$ and 0 otherwise 
($p_t$ is defined in the induced environment $\calE'$, not to be confused with $\bp_t$ in \alg).
Note that $\E[s_t] = 1$ and $\E[s_t^2] \leq \scale$.

\subsection{Contextual Bandits}
Recall that we considered four algorithms:
ILOVETOCONBANDITS, BISTRO+, Epoch-Greedy and EXP4, 
denoted by $\base{1}, \base{2}, \base{3}$ and $\base{4}$ respectively.
Their stability parameters and the corresponding class of environments 
are listed in the following lemmas, followed by the proofs.

\begin{lemma}\label{lem:mini-monster}
If $(x_1, \bc_1), \ldots, (x_T, \bc_T)$ are generated independently
from a fixed and unknown distribution, then \textup{ILOVETOCONBANDITS}
is $(\frac{1}{2}, \otil(\sqrt{KT\ln|\Theta_1|}))$-weakly-stable.
\end{lemma}

\begin{lemma}\label{lem:BISTRO+}
If $x_1, \ldots, x_T$ are generated independently from a fixed and
known distribution, then \textup{BISTRO+} is
$(\frac{1}{3}, \order((KT)^{\frac{2}{3}}(\ln|\Theta_2|)^{\frac{1}{3}}))$-weakly-stable.
\end{lemma}

\begin{lemma}\label{lem:epoch-greedy}
If $(x_1, \bc_1), \ldots, (x_T, \bc_T)$ are generated independently
from a fixed and unknown distribution, then \textup{Epoch-Greedy}
is $(\frac{1}{3}, \otil(T^{\frac{2}{3}} \sqrt{K\ln|\Theta_3|}))$-weakly-stable.
\end{lemma}

\begin{lemma}\label{lem:EXP4}
For any sequence $(x_1, \bc_1), \ldots, (x_T, \bc_T)$, \textup{EXP4}
is $(\frac{1}{2}, \order(\sqrt{KT\ln|\Theta_4|}))$-weakly-stable.
\end{lemma}

To ease notation we drop the subscript for the policy class $\Theta$.
We first point out that the context-loss
sequence that the base algorithm faces in $\calE'$ is $(x_1, \bc'_1), \ldots,
(x_T, \bc'_T)$ where $\bc'_t = s_t \bc_t$. \\

\begin{proof}[\ifcolt\else Proof\fi of Lemma~\ref{lem:mini-monster}]
%Intuitively, the linear loss scaling in Condition~\eqref{eq:weakcond}
%is quite easy to show since rescaling the losses should rescale regret
%by the same factor. However, there is a slight technicality, which is
%
The only technicality here is that the original analysis of ILOVETOCONBANDITS (ILTCB for short)
assumes an i.i.d. loss sequence, while the loss sequence $\bc'_t$ has dependence 
since the sampling probability at a round $t$ can depend on the entire history. 
This is, however, not a problem for
the regret analysis as the martingale structure of the losses
essential to their analysis is preserved.

Indeed, the only lemma involving concentration of the losses in their
regret analysis is Lemma 11 in~\citet{AgarwalHsKaLaLiSc14}. We
reproduce the essential elements of that lemma here in order to
establish the stability condition. Given
the loss function $\bc'_t$ by $\calE'$, the ILTCB algorithm further creates 
a loss estimator $\bc''_t$ so that

\[
\bc''_t(a) = c_t(a_t)s_t
\frac{\one\{a=a_t\}}{Q_t(a)}, 
\]
where $a_t$ is the action picked by ILTCB and $Q_t$ is the probability distribution 
over the actions induced
at round $t$ by ILTCB. Given a policy $\pi$, we now
define the random variable $Z_t := \bc''_t(\pi(x_t)) -
\bc_t(\pi(x_t))$. Letting $H_{t-1}$ denote the history including
everything prior to round $t$, it is easily seen that $\E[Z_t |
  H_{t-1}] = 0$. Furthermore, $Z_t$ is measurable with respect to the
filtration $H_t$ so that it is a martingale difference sequence
adapted to this filtration. It is clear that $|Z_t| \leq \scale
/\mu_t$ (where $\mu_t$ corresponds to $\min_a Q_t(a)$ as in their
analysis). Furthermore, the conditional variance is bounded by
\[
\E[Z_t^2~|~H_{t-1}] \leq \E\left[s_t^2 | H_{t-1}\right]
  \E\left[\frac{\one\{\pi(x_t)=a_t\}}{Q_t(\pi(x_t))^2} \bigg|~ H_{t-1}
  \right] \leq
\scale\mathcal{V}_t(\pi),
\]
where the quantity $\mathcal{V}_t(\pi)$ is as defined in the analysis
of~\citet{AgarwalHsKaLaLiSc14}. Hence, we see that both the range and
second moment of $Z_t$ are scaled by $\scale$. Plugging this into the proof of
their Lemma 11, we see that the RHS of their bound simply becomes
\[
\scale\mathcal{V}_t(\pi)\lambda + \frac{\ln(4t^2|\Pi|/\delta)}{t\lambda}
= \mathcal{V}_t(\pi)\lambda' + \frac{\scale\ln(4t^2|\Pi|/\delta)}{t\lambda'}
\]
where $\lambda \in [0, \mu_t/\scale]$ and $\lambda' \in [0, \mu_t]$.
For the rest of the proof, one can simply replace all $\ln(4t^2|\Pi|/\delta)$
with $\scale\ln(4t^2|\Pi|/\delta)$ and obtain the claimed bound.
%The same scaling propagates through the subsequent Lemma 13,
%and hence Theorem 2. The rescaled version of Theorem 2 of their paper
%is exactly the desired result in Lemma~\ref{lem:mini-monster}, thereby
%completing this proof.
\end{proof}

\begin{proof}[\ifcolt\else Proof\fi of Lemma~\ref{lem:BISTRO+}]
BISTRO+ is a relaxation-based approach. 
For our setting, the relaxation $\Rel$ remains similar:
let $\vrad_t\in \{-1,1\}^K$ be a Rademacher random vector 
(i.e. each coordinate is an independent Rademacher random variable which is $-1$ or $1$ with equal probability), 
$Z_t\in \{0,L\scale\}$ be a random variable which is $L\scale$ with probability $K/(L\scale)$ and $0$ otherwise for some parameter $L$,
and finally let $\xi_t = (x,\epsilon,Z)_{t+1:T}$. Then the new relaxation is defined as follows:
\begin{equation}
\Rel(I_{1:t}) = \E_{\xi_t}\left[ R((x,\hat{c})_{1:t},\xi_t) \right], 
\end{equation}
where
\[ R((x,\hat{c})_{1:t},\xi_t) = -\min_{\theta \in \Theta} \left(\sum_{\tau=1}^t \hat{c}_{\tau, \theta(x_\tau)} + \sum_{\tau=t+1}^T 2\rad_{\tau, \theta(x_\tau)}Z_\tau\right) + (T-t)K/L, \]
\[
\hat{c}_t = L\scale X_t \basis_{\hat{y}_t} \one\{s_t \neq 0\},
\]
\[
X_t = \begin{cases}
1\;   &\text{with probability $\frac{c_{t, \hat{y}_t}}{L\scale p_t q_{t,\hat{y}_t}}$,} \\
0\;   &\text{with the remaining probability,}
\end{cases}
\]
and $I_t$ is the {\it information set} as in~\citep{SyrgkanisLuKrSc16}.
Similarly, one can verify that this modified relaxation 
satisfies the following two admissible conditions:
\[
\E_{x_t} \left[ \min_{q_t \in \Delta_K}\max_{c_t \in [0,1]^K} \E_{\hat{y}_t \sim q_t, X_t, s_t} \left[ s_t c_{t,\hat{y}_t} + \Rel(I_{1:t}) \right] \right]
\leq \Rel(I_{1:{t-1}}),
\]
\[
\E_{\hat{y}_{1:T} \sim q_{1:T}, X_{1:T}, s_{1:T}} \left[ \Rel(I_{1:T}) \right] \geq -\min_{\theta \in \Theta} \sum_{t=1}^T c_{t, \theta(x_t)}
\]
with the following admissible strategy:
\[
q_t = \E_{\xi_t}\left[ q_t(\xi_t)\right]    \text{\quad where \quad} q_t(\xi_t) = \left(1-\frac{K}{L}\right) q_t^*(\xi_t)+\frac{1}{L}{\one}, 
\]
and
\[
q_t^*(\xi_t) = \argmin{q\in \Delta_K}\max_{w_t \in D}\E_{\hat{c}_t\sim w_t}\left[ \inner{q}{\hat{c}_t} + R((x,\hat{c})_{1:t}, \xi_t) \right].
\]
Here, $D$ is a subset of all distributions over $\{\zero, L\scale \basis_{1}, \ldots, L\scale \basis_{K}\}$ such that the mass for each non-zero vector is at most $1/(L\scale)$.
Finally, the expected regret of this modified BISTRO+ is bounded by 
\[
\Rel(\emptyset) \leq \sqrt{TKL\scale\ln|\Theta|} + \frac{TK}{L},
\]
which is $\order((TK)^{\frac{2}{3}}(\scale\ln|\Theta|)^{\frac{1}{3}})$ by choosing the optimal $L$.
This proves the lemma.
\end{proof}

\begin{proof}[\ifcolt\else Proof\fi of Lemma~\ref{lem:epoch-greedy}]
Let $\cbar(\theta) = \E_{(x, \bc)}\left[ c_{\theta(x)} \right]$ be the
expected cost of a policy $\theta$, $\theta^* = \arg\min_{\theta \in
  \Theta} \cbar(\theta)$ be the optimal policy, and $\theta^*_S =
\arg\min_{\theta \in \Theta} \sum_{(x, \bc) \in S} c_{\theta(x)}$ be
the empirically optimal policy with respect to a training set $S$.

For simplicity, consider the following simplest version of
Epoch-Greedy.  For the first $T_0$ rounds (for some parameter $T_0$ to
be specified), actions are chosen uniformly at random.  Then a
training set $S = \{(x_t, \bc''_t)\}_{t=1,\ldots, T_0}$ where
$\bc''_t$ is the usual importance weighted estimator of $\bc'_t$ is
constructed and fed to the ERM oracle to obtain $\theta^*_S$.  Finally
$\theta^*_S$ is used for the rest of the game.

Now for a fixed policy $\theta$, consider the random variable $Y_t = \cbar(\theta) - c''_{t, \theta(x_t)}$.
%First note that $|Y_t| \leq \scale % \leq \sqrt{2TM\scale}$ where the second inequality is due to $\scale \leq 2TM$ 
%(see the proof of ).
It is clear that $|Y_t| \leq K\scale$, $Y_1, \ldots, Y_{T_0}$ form a martingale difference sequence and 
\[
\E_t\left[Y_t^2 \right]\leq 2\left(\cbar^2(\theta) + \E_t\left[(c''_{t, \theta(x_t)})^2 \right]\right)  \leq 2(1 + K\scale).
\]
Therefore by Freedman's inequality for martingales (we use the version of~\citep[Lemma~9]{AgarwalHsKaLaLiSc14} and 
pick $\lambda = \min\{ \frac{1}{K\scale}, \sqrt{\frac{\ln\frac{1}{\delta}}{T_0K\scale}}\}$), 
we have with probability at least $1 - \delta$,
\[
\left|T_0\cbar(\theta) - \sum_{t=1}^{T_0} c''_{t, \theta(x_t)} \right| = \left|\sum_{t=1}^{T_0} Y_t \right| 
\leq \order\left( K\scale \ln\tfrac{1}{\delta}+ \sqrt{T_0K\scale\ln\tfrac{1}{\delta}} \right)
\]
%where the last step is by realizing that and $M$ is a constant here.
A union bound then implies that with probability $1 - \delta$ and notation $B = K\scale\ln\left(\tfrac{|\Theta|}{\delta}\right)/T_0$,
\[
\left| \cbar(\theta^*_S) - \frac{1}{T_0}\sum_{t=1}^{T_0} c''_{t, \theta^*_S(x_t)} \right| 
\leq \order\left(B + \sqrt{B} \right),
\]
and thus with probability at least $1 - 2\delta$, we have by construction of $\theta^*_S$
\begin{align*}
\cbar(\theta^*_S) - \cbar(\theta^*) 
\leq \cbar(\theta^*_S) - \frac{1}{T_0} \sum_{t=1}^{T_0} c''_{t, \theta^*_S(x_t)} - \left(\cbar(\theta^*) - \frac{1}{T_0} \sum_{t=1}^{T_0} c''_{t, \theta^*(x_t)} \right) 
= \order\left(B + \sqrt{B} \right).
\end{align*}
%Solving for $\cbar(\theta^*_S)$ leads to
%\[
%\cbar(\theta^*_S) - \cbar(\theta^*)  \leq \order\left(B + \sqrt{\cbar(\theta^*)B} \right).
%\]
Therefore, setting $\delta = 1/T$ shows that the expected regret of Epoch-Greedy is at most
\[
2 + T_0 + \order\left(\left(\frac{K\scale\ln\left(T|\Theta|\right) }{T_0} + 
\sqrt{\frac{K\scale\ln\left(T|\Theta|\right) }{T_0}} \right) T\right).
\]
Further picking $T_0 = T^{\frac{2}{3}}\scale^{\frac{1}{3}}\sqrt{K\ln(T|\Theta|)}$ leads to
\[
\order\left((T^{\frac{2}{3}}\scale^{\frac{1}{3}} + T^{\frac{1}{3}}\scale^{\frac{2}{3}})\sqrt{K\ln(T|\Theta|)} \right).
\]
Finally note that we in fact only care about $\scale = \order(T)$ (see the proof of Lemma~\ref{lem:master2base}) 
and thus $\scale^{\frac{2}{3}} \leq \order(\scale^{\frac{1}{3}} T^{\frac{1}{3}})$, which
simplifies the above bound to $\order\left(T^{\frac{2}{3}}\scale^{\frac{1}{3}}\sqrt{K\ln(T|\Theta|)}\right)$.
\end{proof}

\begin{proof}[\ifcolt\else Proof\fi of Lemma~\ref{lem:EXP4}]
By standard analysis (see~\cite{AuerCeFrSc02}), the expected regret of EXP4 as a base algorithm is at most
\[
\eta' \E\left[\sum_{t=1}^T\sum_{a=1}^K (c'_{t,a})^2\right] + \frac{\ln|\Theta|}{\eta'},
\]
where $\eta'$ is the internal learning rate parameter of EXP4.
Noting that $\E[(c'_{t,a})^2] \leq \E[s_t^2] \leq \scale$ and picking the optimal $\eta'$ complete the proof.
\end{proof}

\subsection{Convex Bandits}
\begin{lemma}\label{lem:SCRiBLe}
If $f_t(\theta, x) = \inner{\theta}{\bc_t}$ for some $\bc_t \in \R^d$,
then \textup{SCRiBLe} is $(\frac{1}{2}, \otil(d^{\frac{3}{2}}\sqrt{T}))$-weakly-stable.
\end{lemma}

\begin{proof}
The linear loss that the base algorithm SCRiBLe faces is $\inner{\theta}{s_t \bc_t}$ for $t = 1, \ldots, T$.
According to the proof of~\citep[Theorem~5.1]{AbernethyHaRa12}, the expected regret of SCRiBLe as a base algorithm is at most
\[
2\eta' d^2  \E\left[ \sum_{t=1}^T (\inner{\theta_t^1}{s_t \bc_t})^2 \right] + \frac{d\ln T}{\eta'} 
\]
where $\eta'$ is the internal learning rate parameter of SCRiBLe and $\theta_1^1, \ldots, \theta_T^1$ are the decisions of SCRiBLe.
Now noting that $\inner{\theta_t^1}{\bc_t} \leq 1$ and $\E[s_t^2] \leq \scale$ and picking the optimal $\eta'$ give the final regret bound $\otil(d^{\frac{3}{2}}\sqrt{T\scale})$.
\end{proof}

\begin{lemma}\label{lem:BGD}
If $f_1, \ldots, f_T$ are $L$-Lipschitz, then \textup{BGD} is 
$(\frac{1}{4}, \order(d\sqrt{L}T^{\frac{3}{4}}))$-weakly-stable.
\end{lemma}

\begin{proof}
BGD is essentially running a stochastic version of online gradient descent with gradient estimators $g_1, \ldots, g_T$.
The key component in its analysis is Lemma~3.1 of~\citep{FlaxmanKaMc05}, which gives a regret bound of order 
$\sqrt{\sum_{t=1}^T \E[\norm{g_t}^2]} \leq G\sqrt{T}$ for stochastic online gradient descent where $G > 0$ is a bound on $\norm{g_t}$.
When run with a master, BGD uses gradient estimators $s_1 g_1, \ldots, s_T g_T$. 
Since $\E[s_t] \leq \scale$, it is clear that the regret bound for the corresponding stochastic online gradient descent now becomes 
$\sqrt{\sum_{t=1}^T \E[\norm{s_t g_t}^2]} \leq G\sqrt{T\scale}$.
The rest of the analysis remains the same as~\citep{FlaxmanKaMc05}.
\end{proof}

\subsection{Multi-Armed Bandits}

\begin{lemma}\label{lem:TS}
If Thompson Sampling is run with the true prior $\calP$, then it is 
$(\frac{1}{2}, \order(\sqrt{TKH(\bq)}))$-weakly-stable.
\end{lemma}

\begin{proof}
One slight modification here is that the master needs to pass $s_t$ to the Thompson Sampling (TS) strategy in order to 
update the posterior distribution $\bu_t$ of the loss $\bc_t$ (and it is clear that when $s_t=0$ no update happens).
Let $\bc'_t = s_t \bc_t$, $\bq_t$ be the posterior distribution of the optimal arm,
and $v_{t,\theta} = \sum_{j=1}^K q_{t, j} \left (\E_{t}\left[c'_{t,\theta} | \theta^* = j \right] 
-  \E_{t}\left[c'_{t,\theta} \right] \right)^2$ where $\E_t$ denotes the expectation conditional on everything up to time $t$ and $\theta^*$ denotes the optimal arm.
One key modification of the analysis of~\citep{RussoVa14} is to realize that 
\[
v_{t,\theta} \leq \scale \sum_{j=1}^K q_{t, j} \left (\E_{c_t}\left[c_{t,\theta} | \theta^* = j \right] 
-  \E_{c_t}\left[c_{t,\theta} \right] \right)^2,
\]
which is then used to show that (with $\theta_t$ being the output of TS)
\[
\sum_{t=1}^T \E\left[v_{t, \theta_t} \right] \leq \frac{\scale}{2} H(\bq)
\]
by the exact same argument of the original analysis.
The rest of the analysis remains the same.
\end{proof}

\end{document}